\documentclass{scrartcl}
\usepackage{amsmath}
\usepackage{amsfonts}
\usepackage{amssymb}
\usepackage{dsfont}
\usepackage[T1]{fontenc}
\usepackage[latin1]{inputenc}
\usepackage{epsfig}
\usepackage{graphicx}
\usepackage{amsthm}

\usepackage[boxed, lined]{algorithm2e}
\usepackage{float}
\usepackage{comment}
\usepackage{mathtools}

\usepackage{color}

\usepackage[maxbibnames=100, minnames=1,  backend=bibtex,style=authoryear,sorting=anyt,bibstyle=numeric, natbib]{biblatex} 

\usepackage{hyperref}
\usepackage{xurl}
\hypersetup{breaklinks=true}

\bibliography{Literatur.bib}

\newcommand{\bw}{\mathbf{w}}
\newcommand{\bv}{\mathbf{v}}

\newcommand{\D}{{\mathcal{D}}}

\newcommand{\Nu}{{\mathcal{N}}}
\newcommand{\W}{\mathcal{W}}

\newcommand{\N}{\mathbb{N}}

\newcommand{\R}{\mathbb{R}}

\newcommand{\Rd}{\mathbb{R}^d}

\theoremstyle{plain}
\newtheorem{theorem}{Theorem}

\newtheorem{lemma}{Lemma}
\newtheorem{definition}{Definition}
\newtheorem{remark}{Remark}

\newcommand{\EXP}{{\mathbf E}}
\newcommand{\PROB}{{\mathbf P}}

\newcommand{\bf}{\normalfont \bfseries}
\newcommand{\it}{\normalfont \itshape}

\allowdisplaybreaks
\begin{document}
	\newcounter{const}\setcounter{const}{0}
	\newcommand{\nconst}{\stepcounter{const}c_{\arabic{const}}}
	\newcommand{\const}{c_{\arabic{const}}}
	\newcommand{\mconst}{\addtocounter{const}{-1}c_{\arabic{const}}\stepcounter{const}}
	\newcommand{\mmconst}[1]{\addtocounter{const}{-#1}c_{\arabic{const}}\addtocounter{const}{#1}}
	
	\renewcommand{\thefootnote}{\fnsymbol{footnote}}
	\newcommand{\F}{{\cal F}}
	\newcommand{\Sp}{{\cal S}}
	\newcommand{\G}{{\cal G}}
	\newcommand{\HH}{{\cal H}}

	\begin{center}
		
		{\LARGE \bf
			Analysis of the expected $L_2$ error of an over-parametrized
			deep neural network estimate learned by gradient descent
			without regularization
		}
		\footnote{
			Running title: {\it Over-parametrized deep neural network estimates}}
		\vspace{0.5cm}

		Selina Drews\footnote{Corresponding author. Tel:
			+49-6151-16-23372, Fax:+49-6151-16-23381}
		and
		Michael Kohler
		
		{\it 
			Fachbereich Mathematik, Technische Universit\"at Darmstadt,
			Schlossgartenstr. 7, 64289 Darmstadt, Germany,
			email: drews@mathematik.tu-darmstadt.de, kohler@mathematik.tu-darmstadt.de}

	\end{center}
	\vspace{0.5cm}
	
	\begin{center}
		\today
	\end{center}
	\vspace{0.5cm}

	\noindent
	{\bf Abstract}\\
	Recent results show that estimates defined by over-parametrized deep neural networks learned by applying gradient descent to a regularized empirical $L_2$ risk are universally consistent and achieve good rates of convergence.
	In this paper, we show that the regularization term is not necessary to obtain similar results.
	In the case of a suitably chosen initialization of the network, a suitable number of gradient descent steps, and a suitable step size we show that an estimate without a regularization term is universally consistent for bounded predictor variables. Additionally, we show that if the regression function is H\"older smooth with H\"older exponent $1/2 \leq p \leq 1$, the $L_2$ error converges to zero with a convergence rate of approximately $n^{-1/(1+d)}$.
	Furthermore, in case of an interaction model, where the regression function consists of a sum of H\"older smooth functions with $d^*$ components, a rate of convergence is derived which does not depend on the input dimension $d$.
	
	\vspace*{0.2cm}
	
	\noindent{\it AMS classification:} Primary 62G08; secondary 62G20.

	\vspace*{0.2cm}
	
	\noindent{\it Key words and phrases:}
	interaction models,
	neural networks,
	nonparametric regression,
	over-parametrization,
	rate of convergence,
	universal consistency.

	\section{Introduction}
	\label{se1}
	
	\subsection{Deep learning}
	In the last decade, deep learning, i.e. the application of deep neural networks to data, has gained a lot of attraction in practical applications as well as in theoretical considerations and achieves impressive results. 
	For instance, in natural science AlphaFold can predict protein structures (c.f. \citet{Bi19})
	or AlphaZero is able to outperform humans in three popular board games (c.f. \citet{Mc22}).
	Furthermore, the chatbot ChatGPT is one of the most capable language models nowadays, consisting of 175 billion parameters (c.f. \citet{ZoKr22}). 
	Another example is the text-to-image model DALL-E-2, which consists of 3.5 billion parameters (c.f. \citet{Hu22}).
	These observations suggest that a large number of parameters is very useful in practical applications of deep learning.\\
	From a theoretical point of view this great success cannot yet be fully explained. However, there exist already some results concerning deep neural networks. For example, results on the rate of convergence of least squares estimates based on deep neural networks have  been derived (cf., e.g., \citet{BaKo19}, \citet{SchHi20}, \citet{KoLa21}). 
	
	Nevertheless, the above results neglect two features which have turned out to be very important for practical applications.
	First, in contrast to the results above, in practice estimates are computed using gradient descent instead of least squares.
	The second property they ignore is that practical applications often use over-parametrized neural networks. A network is called over-parametrized if the number of its parameters is much larger than the sample size.
	
	In this paper we consider a suitable over-parametrized deep neural network estimate computed by gradient descent and analyze its statistical performance in a nonparametric regression setting.

	\subsection{Nonparametric regression}
	We analyze deep neural network estimates in the context of nonparametric regression.
	To do this we consider a random $\R^d\times \R$-valued vector $(X,Y)$ with $\EXP Y^2 < \infty$. Here we are interested in the dependence of the so-called \textit{response variable} $Y$ on the value of the \textit{observation vector} $X$.
	We assume that it is possible to observe data of $(X,Y)$. This dataset is given by
	\[
	\mathcal{D}_n=\{(X_1,Y_1),\dots, (X_n,Y_n)\}
	\]
	where $(X,Y),(X_1,Y_1),\dots, (X_n,Y_n)$ are independent and identically distributed (i.i.d.).
	Our aim is to construct an estimate $m_n(x) := m_n(x,\mathcal{D}_n)$ of the corresponding regression function $m:\R^d\rightarrow\R$ with $m(x)=\EXP\{Y|X=x\}$ such that the $L_2$ error
	\[
	\int |m_n(x)-m(x)|^2 \PROB_X(dx)
	\]
	is small (cf. e.g. \citet{GyKoKrWa02} for a detailed introduction to nonparametric regression).
	
	There exist different modes of convergence. An important property an estimate should satisfy is to be universally consistent.
	Universal consistency means that the $L_2$ error converges to zero for all distributions of $(X,Y)$. But it does not provide any information about a rate of convergence of the $L_2$ error.
	As shown in Theorem 7.2 and Problem 7.2 in \citet{DeGyLu96} a rate of convergence cannot be derived in general since there always exists a distribution such that the $L_2$ error converges to zero arbitrarily slowly.
	Thus, in order to derive non-trivial results about the rate of convergence, we need to restrict the class of regression functions. To do this we use the following definition of $(p,C)$ smoothness.

	\begin{definition}
		\label{se1de1}
		Let $p=q+s$ for some $q \in \N_0$ and $0< s \leq 1$.
		A function $m:\R^d \rightarrow \R$ is called {\bf $(p,C)$-smooth}, if for every $\alpha=(\alpha_1, \dots, \alpha_d) \in
		\N_0^d$ with $\sum_{j=1}^d \alpha_j = q$ the partial derivative
		$\frac{ \partial^q m}{\partial x_1^{\alpha_1} \dots \partial x_d^{\alpha_d}}$
		exists and satisfies
		\[
		\left|\frac{\partial^q m}{\partial x_1^{\alpha_1}\dots\partial x_d^{\alpha_d}}(x)-\frac{\partial^q m}{\partial x_1^{\alpha_1}\dots\partial x_d^{\alpha_d}}(z)\right|\leq C \cdot \| x-z \|^s
		\]
		for all $x,z \in \R^d$, where $\Vert\cdot\Vert$ denotes the Euclidean norm.
		For $p \leq 1$, the function $m$ is called H\"older-smoooth with exponent $p$ and H\"older-constant $C$.
	\end{definition}

	\citet{Stone82} showed that the optimal minimax rate of convergence in nonparametric regression for $(p,C)$-smooth functions is given by $n^{-2p/(2p+d)}$.
	
	\subsection{Main results}

	The goal of this paper is to study over-parametrized deep neural networks trained by gradient descent from a statistical point of view. 
	
	In this context we define an estimate which fits an over-parametrized deep neural network via gradient descent to the data. The output of the network is defined as a linear combination of a large number of fully connected deep neural networks.
	The main feature of this work is that we do not need a regularization term in the empirical $L_2$ risk.
	To derive the results for the empirical $L_2$ risk without a regularization term, we use a new approach to analyze the optimization error. Previous methods used that the gradient of the empirical $L_2$ risk is Lipschitz continuous (\citet{BrKoLaWa21}) or they used the convexity of the empirical $L_2$ risk (\citet{yesh22}) to analyze the optimization error. In our work, we combine these two techniques (cf. Lemma \ref{le1} below).
	
	Firstly, we prove that for a bounded support of the input data the estimate is universally consistent.
	Secondly, we analyze the rate of convergence of the over-parametrized deep neural network estimate.
	In Theorem \ref{th2} below, we show  that if the regression function is $(p,C)$-smooth with $p \in [1/2,1]$, then the expected $L_2$ error of the truncated estimate tends to zero, with a rate of convergence close to
	\[
	n^{-\frac{1}{1+d}}.
	\]
	Furthermore, in case of an interaction model, i.e. if the regression function is a sum of $(p,C)$-smooth functions where each function depends on at most $d^*$ of the $d$ components of $X$, the estimate achieves a rate of convergence which is close to
	\[
	n^{-\frac{1}{1+d^*}}
	\]
	and does not depend on the input dimension $d$.
	
	\subsection{Discussion of related results}
	From a theoretical point of view, there have been many interesting results in the past  years. For instance, in the case of a suitably defined least squares estimate based on a multilayer neural network, \citet{KoKr17} were able to derive a rate of convergence of $n^{-\frac{2p}{2p+d^*}}$ for $(p,C)$-smooth regression functions, which satisfy a model where the functions consist of a composition of functions applied to at most $d^*$ components of its inputs, with $p\leq 1$. This rate does not depend on the input dimension $d$.
	\citet{BaKo19} were able to show the same assertion for $p>1$, provided that the activation function is sufficiently smooth.
	Under the additional condition of the so-called \textit{sparsity} of the networks, \citet{SchHi20} showed the same rate of convergence for neural network estimates with ReLU activation function in case that the regression function satisfies a kind of \- hierarchical composition model.
	The property of sparsity is not necessary to obtain this rate of convergence for a regression function satisfying a hierarchical composition model (\citet{KoLa21}).
	\citet{Su18} and \citet{SuNi19} were able to show that this dimensional reduction can be achieved even under weaker assumptions on the smoothness of the regression function. For this purpose they considered deep learning with ReLU activation function for functions in Besov spaces.

	In classical machine learning theory, the goal was to avoid over-parametrized neural networks. It was thought that these networks lead to an overfitting of the weights to the data and therefore generalize poorly to new data. Therefore, it was desired to obtain a bias-variance trade-off. The model should be complex enough to represent the structure of the underlying data, but simple enough to avoid overfitting.
	However, in practical applications, over-parametrized neural networks are often used very successfully and achieve high accuracy on new test data. This phenomenon has been studied  recently by many different researchers (c.f., e.g. \citet{BeHsMaMa19}, 	\citet{FrChBa22} and the literature cited therein).
	
	\citet{BeHsMaMa19} were able to reconcile classical theory with new findings in practical applications and identified a pattern of performance dependence on unseen model capacity data and the underlying mechanism of emergence. This dependence is represented by the so-called \textit{double descent curve}. This curve shows that if the model capacity is greater than the interpolation threshold, the test risk will decrease again.

	In recent years, much work has focused on the capacity of neural networks. They tried to understand the ability of neural networks to adapt to the training data either on a finite data set (\citet{BuElLeMi20}) or with respect to neural networks trained by gradient methods (\citet{Da19}, \citet{Da20}) or in neural tangent training (\citet{MoZh20}). In addition \citet{BaLoLuTs20} and \citet{BeRaTs18} have already shown that over-parametrized neural networks can achieve good rates of convergence.

	In practical applications the weights of neural network estimates are computed by gradient descent. So from a theoretical point of view it is very interesting to derive theoretical results for such estimates. \citet{NiSu20} showed that the averaged stochastic gradient descent for over-parametrized neural networks with one hidden layer can achieve the  optimal minimax rate of convergence in a neural tangent kernel setting where the smoothness of the regression function is measured by the neural tangent kernel.
	Furthermore, it was shown that the (stochastic) gradient descent can find a global minimum of the empirical $L_2$ risk for suitable over-parametrized deep neural networks (\citet{AlLiSo18}, \citet{KaHu19}, \citet{ArCoGoHu18}, \citet{DuLeTiPoSi17}).

	However, \citet{KoKr21} were able to exhibit that over-parametrized deep neural networks minimizing the empirical $L_2$ risk do not generalize well on new data, in the sense that the networks that minimize the empirical risk do not achieve the optimal minimax rate of convergence.

	\citet{BrKoLaWa21} showed that under the assumption that the Fourier transform of the regression function decays fast enough, a suitably initialized estimate learned by gradient descent achieves (up to a logarithmic factor) a rate of convergence of $n^{-\frac{1}{2}}$ which does not depend on the input dimension $d$.
	This result is related to a classical result of \citet{Ba94}. He was able to derive a similar rate of convergence for a least squares neural network estimate in case that the Fourier transform has a finite first moment, which requires that the function becomes smoother with increasing dimension.
	\citet{KoKr22b} analyzed the same estimates in an over-parametrized setting and were able to derive an improved rate of convergence of $n^{-\frac{2}{3}}$ for $d=1$ using a suitably regularized $L_2$ risk.
	
	The property of over-parametrization allows the gadient descent to find interpolating solutions that implicitly impose a regularization. This over-parametrization leads to benign overfitting (\citet{BaMoRa21}).
	
	\citet{DrKo22} showed the property of universal consistency for over-param\-etrized deep neural network estimates computed by minimizing the $L_2$ risk with gradient descent.
	For these estimates, that are computed minimizing a regularized $L_2$ risk via gradient descent, \citet{KoKr22a} were able to derive a rate of convergence close to $n^{-\frac{1}{1+d}}$ for suitably smooth regression functions. For this, they considered three key ingredients. 
	Firstly, they used that with high probability a subset of the inner initial weights has good properties. For a suitable choice of the outer weights this causes a good approximation property.
	Secondly, for a suitably chosen number of gradient steps, a suitably chosen step size and suitably chosen bounds for the initial weights, they were able to use a metric entropy bound to control the generalizability of the estimate.
	Thirdly, they analyzed the optimization of the empirical $L_2$ risk by optimizing the outer weights during the gradient descent and by using a regularization term.
	
	\citet{KoKr22a} were also able to show a dimension-independent rate of convergence for interaction models close to $n^{-\frac{1}{1+d^*}}$.
	
	Our results are an extension of the results in \citet{DrKo22} and \citet{KoKr22a}. 
	We show that the regularization term is not necessary. The estimate is universally consistent even by minimizing the empirical $L_2$ risk without a regularization term and achieves similar rates of convergence.

	\subsection{Notation}
	Throughout this paper we will use the following notation:
	The sets of natural numbers, real numbers and non-negative real numbers
	are denoted by $\N$, $\R$ and $\R_+$, respectively. For $z \in \R$, we denote by $\lceil z \rceil$ the smallest integer greater than or equal to $z$ and by $\lfloor z \rfloor$ the largest integer less than or equal to $z$.
	The Euclidean norm of $x \in \Rd$
	is denoted by $\|x\|$.
	For a function $f:\R^d \rightarrow \R$,  we refer to 
	\[
	\|f\|_\infty = \sup_{x \in \R^d} |f(x)|
	\]
	as the supremum norm.
	Let $\F$ be a set of functions $f:\Rd \rightarrow \R$,
	let $x_1, \dots, x_n \in \Rd$, set $x_1^n=(x_1,\dots,x_n)$ and let
	$p \geq 1$.
	A finite collection $f_1, \dots, f_N:\Rd \rightarrow \R$
	is called an $L_p$ $\varepsilon$--cover of $\F$ on $x_1^n$
	if for any $f \in \F$ there exists  $i \in \{1, \dots, N\}$
	such that
	\[
	\left(
	\frac{1}{n} \sum_{k=1}^n |f(x_k)-f_i(x_k)|^p
	\right)^{1/p}< \varepsilon.
	\]
	The $L_p$ $\varepsilon$--covering number of $\F$ on $x_1^n$
	is the  size $N$ of the smallest $L_p$ $\varepsilon$--cover
	of $\F$ on $x_1^n$. It is denoted by $\Nu_p(\varepsilon,\F,x_1^n)$.
	
	If $A$ is a subset of $\Rd$ and $x \in \Rd$, then we set
	$1_A(x)=1$ if $x \in A$ and $1_A(x)=0$ otherwise.

	For $z \in \R$ and $\beta>0$ we define 
	$T_\beta z = \max\{-\beta, \min\{\beta,z\}\}$. If $f:\R^d \rightarrow
	\R$
	is a function and $\F$ is a set of such functions, then we set
	$
	(T_{\beta} f)(x)=
	T_{\beta} \left( f(x) \right)$ and
	\[
	T_{\beta} \F = \{ T_\beta f \, : \, f \in \F \}.
	\]

	\subsection{Outline}  
	
	In Section \ref{se2} we define the estimate. The main results concerning
	the universal consistency and the rate of convergence
	of the deep neural network estimate
	learned by gradient descent are presented in Section \ref{se3}. The proofs of the main results are given in Section \ref{se4}.

	\section{Definiton of the estimate}
	\label{se2}
	In order to define the estimate, let $\sigma(x)=1/(1+e^{-x})$ be the logistic squasher.
	For the neural network topology, we choose a linear combination of $K_n$ fully connected neural networks with $L$ layers and $r$ neurons per layer. The output of the network is defined by
	\begin{equation}\label{se2eq1}
		f_\bw(x) = \sum_{j=1}^{K_n} w_{1,1,j}^{(L)} \cdot f_{j,1}^{(L)}(x) 
	\end{equation}
	for some $w_{1,1,1}^{(L)}, \dots, w_{1,1,K_n}^{(L)} \in \mathbb{R}$, where $f_{j,1}^{(L)}$ is recursively defined by
	\begin{equation}
		\label{se2eq2}
		f_{k,i}^{(l)}(x) = \sigma\left(\sum_{j=1}^{r} w_{k,i,j}^{(l-1)} \cdot f_{k,j}^{(l-1)}(x) + w_{k,i,0}^{(l-1)} \right)
	\end{equation}
	for some $w_{k,i,0}^{(l-1)}, \dots, w_{k,i, r}^{(l-1)} \in \mathbb{R}$
	$(l=2, \dots, L)$
	and
	\begin{equation}
		\label{se2eq3}
		f_{k,i}^{(1)}(x) = \sigma \left(\sum_{j=1}^d w_{k,i,j}^{(0)} \cdot x^{(j)} + w_{k,i,0}^{(0)} \right)
	\end{equation}
	for some $w_{k,i,0}^{(0)}, \dots, w_{k,i,d}^{(0)} \in \mathbb{R}$.

	We denote by $f_{k,i}^{(l)}$ the output of neuron $i$ in the $l$-th layer of the $k$-th fully connected network. By 
	$w_{k,i,j}^{(l-1)}$ we denote the weight between neuron $j$ in the $(l-1)$-th layer and neuron
	$i$ in the $l$-th layer of the $k$-th fully connected network.
	
	The number
	of weights of the neural network is given by
	\[
	W_n
	=
	K_n \cdot (1 + (r+1) +(L-2) \cdot r \cdot (r+1) + r \cdot (d+1)).
	\]
	
	To obtain the estimate, we want to learn the weights using gradient descent. For this we initialize the weights
	$\bw^{(0)}=((\bw^{(0)})_{k,i,j}^{(l)})_{k,i,j,l}
	$ by setting 
	\begin{equation}\label{se2eq5}
		(\bw^{(0)})_{1,1,k}^{(L)}=0
		\quad \text{for } k=1, \dots, K_n,
	\end{equation}
	and by choosing all components of $\bw^{(0)}$ such that they are independent.
	We choose the weights $(\bw^{(0)})_{k,i,j}^{(l)}$ with $l \in \{1, \dots, L-1\}$ uniformly distributed on
	$[-20d \cdot (\log n)^2, 20d \cdot (\log n)^2]$ and 
	$(\bw^{(0)})_{k,i,j}^{(0)}$ uniformly distributed on
	$[- 8d \cdot (\log n)^2 \cdot n^\tau, 8d \cdot (\log n)^2 \cdot n^\tau]$ for some fixed $\tau>0$.

		Set 
		\[
		\lambda_n
		=
		\frac{1}{t_n}
		\]
		and compute
		\[
		\bw^{(t+1)}=\bw^{(t)} - \lambda_{n}\cdot (\nabla_\bw
		F_n)(\bw^{(t)})
		\]
		for $t=0, \dots, t_n-1$
		where
		\[
		F_n(\bw) = \frac{1}{n} \sum_{i=1}^n |f_{\bw}(X_i) -Y_i|^2 
		\]
		is the empirical $L_2$ risk of the network $f_{\bw}$ on the training data.
		The number of gradient descent steps $t_n$ will be chosen in Section \ref{se3} below.
		
		The estimate is then defined by
		\[
		m_{n}(x) = T_{\beta_n} f_{\bw^{(t_n)}} (x),
		\]
		where $\beta_n = \nconst \cdot \log n$.
		
		Because of (\ref{se2eq5}) we have
	\[
	F_n(\bw^{(0)}) = \frac{1}{n} \sum_{i=1}^n |Y_i|^2.
	\]

	\section{Main results}
	\label{se3}
	\subsection{Universal consistency}
	Our first result is the following theorem which presents the universal consistency for bounded $X$ of an estimate learned by gradient descent using the empirical $L_2$ risk without a regularization term.
	
	\begin{theorem}
		\label{th1}  
		Let $\sigma(x)=1/(1+e^{-x})$ be the logistic squasher, and let
		$K_n, L, r \in \N$ and $\tau \in \R_+$.
		Furthermore, assume that $L \geq 2$, $r \geq 2d$, $\tau=1/(d+1)$,
		\begin{equation}
			\label{th1eq1}
			\frac{K_n}{n^\kappa} \rightarrow 0 \quad (n \rightarrow \infty)
		\end{equation}
		for some $\kappa>0$,
		\begin{equation}
			\label{th1eq2}
			\frac{K_n}{n^{r+2}}
			\rightarrow \infty \quad (n \rightarrow \infty), 
		\end{equation}
		and $\beta_n= c_1 \cdot \log n$ for some $\const>0$. Let the estimate $m_{n}$ be defined as in Section \ref{se2} with
		\begin{eqnarray}	\label{th1eq3}
			t_n = \lceil \nconst \cdot L_n  \rceil
		\end{eqnarray}
	for some $\const\geq1$ and $L_n > 0$ which satisfies
	\begin{align*}
		 L_n \geq K_n^{3/2} \cdot (\log n)^{6L+5} .
	\end{align*}
		
		Then we have
		\[
		\EXP \int |m_n(x)-m(x)|^2 \PROB_X (dx)
		\rightarrow 0
		\quad
		(n \rightarrow \infty)
		\]
		for every distribution of $(X,Y)$ where $supp(X)$ is bounded and $\EXP Y^2 < \infty$.
	\end{theorem}
	
	\begin{remark}
		The estimate is over-parametrized in the sense that the number of its parameters is much larger than the sample size. This is due to condition (\ref{th1eq2}), which requires that $K_n$ is asymptotically larger than $n^{r+2} \geq n^{2d+2}$.
		Theorem \ref{th1} shows that using the empirical $L_2$ risk without a regularization term, a suitable large number of steps, and a step size which is equal to the reciprocal of the number of steps provides a good generalization of new independent data.
	\end{remark}
	
	\begin{remark}
		The proof uses that the inner weights are chosen with high probability such that for properly chosen outer weights of the neural networks, the corresponding neural network has a small empirical $L_2$ risk.  Hence, this neural network is based on representation guessing instead of representation learning.
	\end{remark}

	\subsection{Rate of convergence for $(p,C)$-smooth regression functions}
	
	Besides the result about the universal consistency of the estimate, it is interesting how fast the expected $L_2$ error converges to 0. Therefore, in the next theorem a rate of convergence for $(p,C)$-smooth functions is derived.
	\begin{theorem}
		\label{th2}  
		Let $n \in \N$ and let $(X,Y),(X_1,Y_1),\dots, (X_n,Y_n)$ be independent and identically distributed $\R^d\times \R$ valued random variables which satisfy $supp(X)\subseteq [0,1]^d$ and 
		\begin{eqnarray}
			\EXP\left\{e^{\nconst\cdot Y^2}\right\} < \infty
		\end{eqnarray}
		for some $\const>0$.
		Assume that the corresponding regression function $m(x) = \EXP\{Y|X=x\}$ is $(p,C)$-smooth for some $1/2\leq p \leq 1$ and some $C>0$.
		
		Let $\sigma(x)=1/(1+e^{-x})$ be the logistic squasher, let $L,r \in \N$ with $L \geq 2$ and $r \geq 2d$. Set $\beta_n= c_1 \cdot \log n$ for some $c_1>0$,
		\begin{align*}
			K_n& = n^{6d+r+2},\\
			\tau& = \frac{1}{1+d}.
		\end{align*}
		Define the estimate $m_{n}$ as in Section \ref{se2} with
			\begin{eqnarray}	\label{th1eq3}
					t_n = \lceil \nconst \cdot L_n  \rceil
		\end{eqnarray}
		for some $\const\geq1$ and $L_n > 0$ which satisfies
		\begin{align*}
			L_n \geq K_n^{3/2} \cdot (\log n)^{6L+2}
		\end{align*}
		and assume that
		\begin{eqnarray}
			c_1\cdot c_3 \geq 2.
		\end{eqnarray}
		Then we have for any $\epsilon > 0$ 
		\begin{eqnarray*}
			\EXP \int |m_n(x)-m(x)|^2\mathbf{P}_X(dx) \leq \nconst \cdot n^{- \frac{1}{1+d}+\epsilon}.
		\end{eqnarray*}
	\end{theorem}
	
	\begin{remark}
		According to \citet{Stone82} the optimal minimax rate of convergence for $(p,C)$-smooth functions is $n^{-\frac{2p}{2p+d}}$.
		Thus, the rate of convergence derived from Theorem \ref{th2} is almost optimal for $p = \frac{1}{2}$. Unfortunately, if $p > \frac{1}{2}$, the rate of convergence derived above is not almost optimal. We assume that this is not a property of the estimate, but rather a consequence of our proof.
	\end{remark}

	\subsection{Rate of convergence in an interaction model}\label{se3.3}

	The aim of this subsection is to modify the estimate defined above in order to obtain a rate of convergence that does not depend on $d$.
	For this we assume that the regression function satisfies
	\[
	m(x)= \sum_{I \subseteq \{1, \dots, d\} \, : \, |I|=d^* } m_I (x_I),
	\]
	with $1 \leq d^* < d$. The functions $m_I: \R^{d^*} \rightarrow \R$
	$(I \subseteq \{1, \dots, d\}$, $ |I|=d^* )$ are $(p,C)$-smooth
	functions and we use the notation
	\[
	x_I = (x^{(j_1)}, \dots, x^{(j_{d^*})})
	\]
	for $I=\{j_1, \dots, j_{d^*} \}$.

	The neural network is then given by
	\[
	f_{\bw}(x) = \sum_{I \subseteq \{1, \dots, d\} \, : \, |I|=d^* } f_{\bw_I} (x_I)
	\]
	where $f_{\bw_I}$ is defined by (\ref{se2eq1})--(\ref{se2eq3}) with
	$d$ replaced by $d^*$ and weight vector $\bw_I$, where
	\[
	\bw= \left( \bw_I \right)_{I \subseteq \{1, \dots, d\}, |I|=d^* }.
	\]
	We initialize the weights
	$\bw^{(0)}=(((\bw^{(0)}_I)_{k,i,j}^{(l))})_{k,i,j,l})_{I \subseteq \{1, \dots, d\}, |I|=d^* }
	$ by setting 
	\[
	(\bw^{(0)}_I)_{1,1,k}^{(L)}=0
	\quad (k=1, \dots, K_n,  I \subseteq \{1, \dots, d\},  |I|=d^* )
	\]
	and by choosing all weights $(\bw^{(0)}_I)_{k,i,j}^{(l)}$ such that they are independent. We choose the weights $(\bw^{(0)}_I)_{k,i,j}^{(l)}$ with $l \in \{1, \dots, L-1\}$ uniformly distributed on the interval ${[-20d^* \cdot (\log n)^2, 20d^* \cdot (\log n)^2]}$ and we
	choose
	$(\bw^{(0)}_I)_{k,i,j}^{(0)}$ uniformly distributed on the interval
	${[-8d \cdot (\log n)^2 \cdot n^\tau,8d \cdot (\log n)^2 \cdot n^\tau]}$ where $\tau=\frac{1}{1+d^*}$
	($I \subseteq \{1, \dots, d\}$ with $|I|=d^*$).
	
	Similar as above we define the estimate as follows:
	Set 
		\[
		\lambda_n
		=
		\frac{1}{t_n}
		\]
		and compute
		\[
		\bw^{(t+1)}=\bw^{(t)} - \lambda_{n}\cdot (\nabla_\bw
		F_n)(\bw^{(t)})
		\]
		for $t=0, \dots, t_n-1$
		where
		\[
		F_n(\bw) = \frac{1}{n} \sum_{i=1}^n |f_{\bw}(X_i) -Y_i|^2 
		\]
		is the empirical $L_2$ risk of the network $f_{\bw}$ on the training data.
		The number  of gradient descent steps $t_n$ will be chosen in Theorem \ref{th3} below.
		
		The estimate is then defined by
		\[
		m_{n}(x) = T_{\beta_n} f_{\bw^{(t_n)}} (x)
		\]
		where $\beta_n = c_1 \cdot \log n$.
		
	\begin{theorem}
		\label{th3}  
		Let $d \in \N, d^* \in \{1,\dots,d\}, \ 1/2 \leq p \leq 1$. Furthermore, let $C>0$, $n\in \N$ and let $(X,Y),(X_1,Y_1),\dots,(X_n,Y_n)$ be independent and identically distributed $\R^d\times \R$ valued random variables such that $supp(X) \subseteq [0,1]^d$ and
		\[
		\EXP\{e^{c_3\cdot Y^2}\} < \infty
		\]
		holds for some $\mmconst{2} > 0$.
		Assume that the corresponding regression function $m(x) = \EXP\{Y|X=x\}$ satisfies
		\[
		m(x) = \sum_{I \subseteq \{1,\dots,d\}:|I|=d^*} m_I(x_I) \qquad (x \in [0,1]^d)
		\]
		for some $(p,C)$-smooth functions $m_I:\R^{d^*} \rightarrow \R$ ($I\subseteq \{1,\dots,d\},|I|=d^*$).
		
		Let $\sigma(x)=1/(1+e^{-x})$ be the logistic squasher. Let $L,r \in \N$ with $L \geq 2$ and $r \geq 2d^*$. Set $\beta_n= c_1 \cdot \log n$,
		\begin{align*}
			K_n =n^{6d^*+r+2} \qquad \text{and} \qquad 
			\tau = \frac{1}{1+d^*}.
		\end{align*}
		Define the estimate $m_{n}$ as in Section \ref{se3.3} with 
		\begin{eqnarray*}
				t_n = \lceil \nconst \cdot L_n\rceil 
		\end{eqnarray*}
	for some $\const\geq1$ and $L_n > 0$ which satisfies
	\begin{align*}
		L_n \geq K_n^{3/2} \cdot (\log n)^{6L+2}
	\end{align*}
	and assume that
	\begin{eqnarray}
		c_1\cdot c_3 \geq 2.
		\end{eqnarray}
		
		Then we have for any $\epsilon > 0$
		\begin{eqnarray*}
			\EXP \int |m_{n}(x)-m(x)|^2\mathbf{P}_X(dx) \leq \nconst \cdot n^{- \frac{1}{1+d^*}+\epsilon}.
		\end{eqnarray*}
	\end{theorem}
	
	\begin{remark}
		The optimal minimax rate of convergence $n^{-\frac{2p}{2p+d}}$ derived by \citet{Stone82} suffers from the curse of dimensionality. This means that if d is very large compared to p, the rate of convergence becomes extremely slow. However, \citet{Stone94} has already established that under appropriate assumptions it is possible to circumvent the curse of dimensionality.
		In Theorem \ref{th3}, we were able to show that under the above assumptions on the regression function, the over-parametrized neural network estimate can also circumvent the curse of dimensionality.

	\end{remark}
	
	\section{Proofs}
	\label{se4}

	\subsection{Auxiliary results for the proof of Theorem \ref{th1}}
	\label{se4sub1}
	
	In this section we will present auxiliary results that are necessary for the proof of Theorem \ref{th1}. The first auxiliary result enables us to analyze the gradient descent.
	
		\begin{lemma}
			\label{le1}
			Let $d_1, d_2 \in \N$,
			let $(u_0,v_0) \in \R^{d_1} \times \R^{d_2}$,
			let $F:\R^{d_1} \times \R^{d_2} \rightarrow \R_+$
			be a continuously differentiable function and set
			 \[
			A=  \left\{
			(u,v) \in \R^{d_1} \times \R^{d_2} \, : \,
			\|(u,v) - (u_0,v_0)\| \leq
			2 \cdot \sqrt{F(u_0,v_0)} + 1
			\right\}.
			\]
			Let $u^* \in \R^{d_1}$ and assume that for some $D_n,L_n >0$ it holds
			\[
			u \mapsto F(u,v) \quad \mbox{is convex for all } v \in \R^{d_2},
			\]
			\begin{equation}
				\label{le1eq1}
				\| (\nabla_{(u,v)} F)(u,v) \| \leq L_n
			\end{equation}
			for all $(u,v) \in A$,
			\begin{equation}
				\label{le1eq2}
				\| (\nabla_{(u,v)} F)(u_1,v_1) - (\nabla_{(u,v)} F)(u_2,v_2) \| \leq
				L_n \cdot \| (u_1,v_1)-(u_2,v_2) \|
			\end{equation}
			for all $(u_1,v_1),(u_2,v_2) \in A$ 
			and
			\begin{equation}
				\label{le1eq3}
				|F(u^*,v)-F(u^*,v_0)| \leq D_n \cdot \|u^*\| \cdot \|v-v_0\|
			\end{equation}
			for all $v \in \{\tilde{v} \, : \, \|\tilde{v}-v_0\| \leq \sqrt{2 \cdot F(u_0,v_0)} \}$.
			
			Set
			\[
			u_{t+1} = 
			u_t - \lambda \cdot \left( \nabla_u F \right)(u_t,v_t),
			\]
			\[
			v_{t+1} = 
			v_t - \lambda\cdot \left( \nabla_v F \right)(u_t,v_t)
			\]
			for $t=0,1, \dots, t_n-1$, where
			\[
			t_n \geq L_n \qquad \mbox{and} \qquad \lambda = \frac{1}{t_n}.
			\]
			Then we have 
			\begin{eqnarray*}
				&&
				F(u_{t_n},v_{t_n})
				\leq
				F(u^*,v_0)
				+
				D_n \cdot \|u^*\| \cdot \sqrt{2 \cdot F(u_0,v_0)}
				+
				\frac{\|u^*-u_0\|^2}{2}
				+
				\frac{F(u_0,v_0)}{t_n}.
			\end{eqnarray*}
		Furthermore, even if (\ref{le1eq3}) does not hold, we have
			\begin{align*}
		\|u_{t}-u_0\| \leq \sqrt{2 \cdot F(u_0,v_0)}
		\quad \mbox{and} \quad
		\|v_{t}-v_0\| \leq \sqrt{2 \cdot F(u_0,v_0)}
		\end{align*}
		for all $t=0,1, \dots, t_n$.
		\end{lemma}
		
		\begin{proof}
			In the {\it first step of the proof} we show
			\begin{equation}
				\label{ple1eq1}
				\frac{1}{t_n}
				\sum_{t=0}^{t_n-1} F(u_t,v_t)
				\leq
				\frac{1}{t_n}
				\sum_{t=0}^{t_n-1} F(u^*,v_t)
				+
				\frac{\|u^*-u_0\|^2}{2}
				+
				\frac{1}{2 \cdot t_n}
				\sum_{t=0}^{t_n-1}
				\lambda \cdot
				\| (\nabla_{u}F)(u_t,v_t) \|^2.
			\end{equation}
			By convexity of $u \mapsto F(u,v_t)$ we have
			\begin{eqnarray*}
				&&
				F(u_t,v_t) - F(u^*,v_t)
				\\
				&&
				\leq
				\langle (\nabla_u F)(u_t,v_t), u_t-u^* \rangle
				\\
				&&
				=
				\frac{1}{2 \cdot \lambda}
				\cdot 2
				\cdot
				\langle \lambda \cdot (\nabla_u F)(u_t,v_t), u_t-u^* \rangle
				\\
				&&
				\leq
				\frac{1}{2 \cdot \lambda}
				\cdot
				\left(
				- \| u_t - u^* - \lambda \cdot (\nabla_u F)(u_t,v_t)\|^2
				+
				\|u_t-u^*\|^2
				+
				\|  \lambda \cdot (\nabla_u F)(u_t,v_t) \|^2
				\right)\\
				&&
				=
				\frac{1}{2 \cdot \lambda}
				\cdot
				\left(
				\|u_t-u^*\|^2
				- \| u_{t+1} - u^*\|^2
				+
				\lambda^2 \cdot \|  (\nabla_u F)(u_t,v_t) \|^2
				\right).
			\end{eqnarray*}
			This implies
			\begin{eqnarray*}
				&&      \frac{1}{t_n}
				\sum_{t=0}^{t_n-1} F(u_t,v_t)
				-
				\frac{1}{t_n}
				\sum_{t=0}^{t_n-1} F(u^*,v_t)
				\\
				&&
				=
				\frac{1}{t_n}
				\sum_{t=0}^{t_n-1} \left( F(u_t,v_t)
				-
				F(u^*,v_t) \right)
				\\
				&&
				\leq
				\frac{1}{t_n}
				\sum_{t=0}^{t_n-1}
				\frac{1}{2 \cdot \lambda}
				\cdot
				\left(
				\|u_t-u^*\|^2
				- \| u_{t+1} - u^*\|^2
				\right)
				+
				\frac{1}{t_n}
				\sum_{t=0}^{t_n-1}
				\frac{\lambda}{2} \cdot \|  (\nabla_u F)(u_t,v_t) \|^2
				\\
				&&
				=
				\frac{1}{2}
				\cdot
				\sum_{t=0}^{t_n-1}
				\left(
				\|u_t-u^*\|^2
				- \| u_{t+1} - u^*\|^2
				\right)
				+
				\frac{1}{2 \cdot t_n}
				\sum_{t=0}^{t_n-1}
				\lambda \cdot \|  (\nabla_u F)(u_t,v_t) \|^2
				\\
				&&
				\leq
				\frac{ \|u_0-u^*\|^2}{2}
				+
				\frac{1}{2 \cdot t_n}
				\sum_{t=0}^{t_n-1}
				\lambda \cdot \|  (\nabla_u F)(u_t,v_t) \|^2. 
			\end{eqnarray*}
			
			In the {\it second step of the proof} we show that
			\begin{eqnarray}
				\label{ple1eq2}
				&&
				\| (\nabla_{(u,v)} F)((u_t,v_t) + \tau \cdot ((u_{t+1},v_{t+1})-(u_t,v_t)))
				- (\nabla_{(u,v)} F)(u_t,v_t)\| \nonumber \\
				&&
				\leq
				L_n \cdot \tau \cdot \|(u_{t+1},v_{t+1})-(u_t,v_t)\|
			\end{eqnarray}
			for all $\tau \in [0,1]$ implies
			\begin{eqnarray}
				\label{ple1eq3}
			F(u_{t+1},v_{t+1})-F(u_t,v_t)
			\leq
			-\frac{1}{2} \cdot \lambda \cdot \| (\nabla_u F)(u_t,v_t) \|^2
			-\frac{1}{2} \cdot \lambda \cdot \| (\nabla_v F)(u_t,v_t) \|^2.
			\end{eqnarray}
			
			The function
			\[
			H: [0,1] \rightarrow \R, \quad
			H(\tau)=F((u_t,v_t)+ \tau \cdot ((u_{t+1},v_{t+1})-(u_t,v_t)))
			\]
			is continuously differentiable.
			By the fundamental theorem of calculus, assumption (\ref{ple1eq2}) and $
			\lambda \leq 1/L_n$ we get
			\begin{eqnarray*}
				&&
				F(u_{t+1},v_{t+1})-F(u_t,v_t)
				= H(1)-H(0)
				= \int_0^1 H^\prime(\tau) d \tau
				\\
				&&
				=
				\int_0^1
				(\nabla_{(u,v)} F)( (u_t,v_t)+ \tau \cdot ((u_{t+1},v_{t+1})-(u_t,v_t)))
				\cdot ((u_{t+1},v_{t+1})-(u_t,v_t)) d \tau
				\\
				&&
				=\int_0^1
				\left(
				(\nabla_{(u,v)} F)( (u_t,v_t)+ \tau \cdot ((u_{t+1},v_{t+1})-(u_t,v_t)))
				- (\nabla_{(u,v)} F)(u_t,v_t)
				\right)
				\\
				&&
				\hspace*{3cm}
				\cdot  ((u_{t+1},v_{t+1})-(u_t,v_t)) d \tau
				\\
				&&
				\quad
				+ \int_0^1
				(\nabla_{(u,v)} F)(u_t,v_t)
				\cdot ((u_{t+1},v_{t+1})-(u_t,v_t)) d \tau
				\\
				&&
				\leq
				\int_0^1
				\left\|
				(\nabla_{(u,v)} F)( (u_t,v_t)+ \tau \cdot ((u_{t+1},v_{t+1})-(u_t,v_t)))
				- (\nabla_{(u,v)} F)(u_t,v_t)
				\right\|
				\\
				&&
				\hspace*{3cm}
				\cdot  \|(u_{t+1},v_{t+1})-(u_t,v_t)\| d \tau
				\\
				&&
				\quad
				+ 
				(\nabla_{(u,v)} F)(u_t,v_t)
				\cdot ((u_{t+1},v_{t+1})-(u_t,v_t))
				\\
				&&
				\leq
				\int_0^1
				L_n \cdot \tau \cdot  \|(u_{t+1},v_{t+1})-(u_t,v_t)\|^2 d \tau
				\\
				&&
				\quad
				+ 
				(\nabla_{(u,v)} F)(u_t,v_t)
				\cdot ((u_{t+1},v_{t+1})-(u_t,v_t))
				\\
				&&
				=
				\frac{1}{2} \cdot L_n \cdot  \|(u_{t+1},v_{t+1})-(u_t,v_t)\|^2
				+ 
				(\nabla_{(u,v)} F)(u_t,v_t)
				\cdot ((u_{t+1},v_{t+1})-(u_t,v_t))
				\\
				&&
				=\frac{1}{2} \cdot L_n \cdot
				\left(\lambda^2 \cdot \|(\nabla_{u} F)(u_t,v_t)\|^2 + \lambda^2
				\cdot \|(\nabla_{v} F)(u_t,v_t)\|^2 \right)
				- \lambda \cdot \|(\nabla_{u} F)(u_t,v_t)\|^2
				\\
				&&
				\quad
				- \lambda
				\cdot \|(\nabla_{v} F)(u_t,v_t)\|^2
				\\
				&&
				=
				\lambda \cdot \left(\frac{1}{2} \cdot L_n \cdot \lambda -1\right)
				\cdot \|(\nabla_{u} F)(u_t,v_t)\|^2
				+
				\lambda \cdot \left(\frac{1}{2} \cdot L_n \cdot  \lambda-1\right)
				\cdot \|(\nabla_{v} F)(u_t,v_t)\|^2
				\\
				&&
				\leq
				-\frac{1}{2} \cdot \lambda \cdot \| (\nabla_u F)(u_t,v_t) \|^2
				-\frac{1}{2} \cdot \lambda\cdot \| (\nabla_v F)(u_t,v_t) \|^2.
			\end{eqnarray*}
			
			In the {\it third step of the proof} we show
			\[
			F(u_1,v_1) - F(u_0,v_0) \leq
			-\frac{1}{2} \cdot \lambda \cdot \| (\nabla_u F)(u_0,v_0) \|^2
			-\frac{1}{2} \cdot \lambda \cdot \| (\nabla_v F)(u_0,v_0) \|^2.
			\]
			By (\ref{le1eq1}) we know
			\[
			\|(\nabla_{(u,v)} F)(u_0,v_0)\| \leq L_n, 
			\]
			which implies for any $ \tau \in [0,1]$
			\begin{eqnarray*}
				&&
				\|(u_0,v_0) + \tau \cdot ((u_{1},v_{1})-(u_0,v_0))-(u_0,v_0)\|
				\leq \lambda \cdot L_n.
			\end{eqnarray*}
			Consequently we can conclude from
			(\ref{le1eq2}) that (\ref{ple1eq2}) holds for $t=0$,
			from which we get the assertion of step 3 by applying the result
			from step 2.
			
		In the {\it fourth step of the proof} we show that
		by induction on $t$,
			that
			\[
			F(u_{t+1},v_{t+1})-F(u_t,v_t)
			\leq
			-\frac{1}{2} \cdot \lambda \cdot \| (\nabla_u F)(u_t,v_t) \|^2
			-\frac{1}{2} \cdot \lambda \cdot \| (\nabla_v F)(u_t,v_t) \|^2
			\]
			holds for all $t \in \{0,1, \dots, t_n-1\}$, and that
			\[
			\|u_t-u_0\| \leq \sqrt{2 \cdot F(u_0,v_0)}
			\quad \mbox{and} \quad
			\|v_t-v_0\| \leq \sqrt{2 \cdot F(u_0,v_0)}
			\]
			hold for all $t \in \{0, 1, \dots, t_n\}$.
			
			For $t=0$ the assertion follows from step 3. So assume now that
			the assertion holds for some $t \in \{0,1, \dots, t_n-1\}$.
			Then (\ref{le1eq1}) implies
			\[
			\|(\nabla_{(u,v)} F)(u_t,v_t)\| \leq L_n, 
			\]
			which implies for any $ \tau \in [0,1]$
			\begin{eqnarray*}
				&&
				\|(u_t,v_t) + \tau \cdot ((u_{t+1},v_{t+1})-(u_t,v_t))-(u_0,v_0)\|
				\leq \|(u_t,v_t)-(u_0,v_0)\| + \lambda \cdot L_n \\
				&&
				\leq
				\sqrt{\|u_t-u_0\|^2 + \|v_t-v_0\|^2}
				+  \lambda \cdot L_n
				\leq 2 \cdot \sqrt{ F(u_0,v_0)}+  \lambda \cdot L_n.
			\end{eqnarray*}
			Consequently we can conclude from
			(\ref{le1eq2}) that (\ref{ple1eq2}) holds,
			from which we get by step 2
			\[
			F(u_{t+1},v_{t+1})-F(u_t,v_t)
			\leq
			-\frac{1}{2} \cdot \lambda \cdot \| (\nabla_u F)(u_t,v_t) \|^2
			-\frac{1}{2} \cdot \lambda \cdot \| (\nabla_v F)(u_t,v_t) \|^2.
			\]
			Furthermore, we have
			\begin{eqnarray*}
				&&
				\|u_{t+1}-u_0\|
				\\
				&&
				\leq
				\sum_{s=0}^t \|u_{s+1}-u_s\|
				\\
				&&
				\leq
				\sqrt{
					(t+1) \cdot \sum_{s=0}^t \|u_{s+1}-u_s\|^2
				}
				\\
				&&
				=
				\sqrt{
					(t+1) \cdot \sum_{s=0}^t \lambda^2 \cdot \|(\nabla_u F)(u_s,v_s)\|^2
				}
				\\
				&&
				\leq   \sqrt{
					2 \cdot (t+1) \cdot \lambda \cdot \sum_{s=0}^t (F(u_s,v_s)-F(u_{s+1},v_{s+1}))
				}
				\\
				&&
				\leq \sqrt{
					2 \cdot (t+1) \cdot \frac{1}{t_n} \cdot F(u_0,v_0)}
				\\
				&&
				\leq
				\sqrt{2  \cdot F(u_0,v_0)}
			\end{eqnarray*}
			and
			\begin{eqnarray*}
				&&
				\|v_{t+1}-v_0\|
				\\
				&&
				\leq
				\sum_{s=0}^t \|v_{s+1}-v_s\|
				\\
				&&
				\leq
				\sqrt{
					(t+1) \cdot \sum_{s=0}^t \|v_{s+1}-v_s\|^2
				}
				\\
				&&
				=
				\sqrt{
					(t+1) \cdot \sum_{s=0}^t \lambda^2 \cdot \|(\nabla_v F)(u_s,v_s)\|^2
				}
				\\
				&&
				\leq   \sqrt{
					2 \cdot (t+1) \cdot \lambda\cdot \sum_{s=0}^t (F(u_s,v_s)-F(u_{s+1},v_{s+1}))
				}
				\\
				&&
				\leq \sqrt{
					2 \cdot (t+1) \cdot \frac{1}{t_n} \cdot F(u_0,v_0)}
				\\
				&&
				\leq
				\sqrt{2  \cdot F(u_0,v_0)}.
			\end{eqnarray*}
			
			In the {\it fifth step of the proof} we show the assertion.
			By the fourth step we know that $F$ is monotonically decreasing hence it holds
			\[
			F(u_{t_n},v_{t_n}) \leq
			\frac{1}{t_n}
			\sum_{t=0}^{t_n-1} F(u_t,v_t).
			\]
			Then we can conclude from the first step of the proof
			\begin{eqnarray*}
				&&
				F(u_{t_n},v_{t_n})
				\\
				&&
				\leq
				\frac{1}{t_n}
				\sum_{t=0}^{t_n-1} F(u^*,v_t)
				+
				\frac{\|u^*-u_0\|^2}{2}
				+
				\frac{1}{2 \cdot t_n}
				\sum_{t=0}^{t_n-1}
				\lambda \cdot
				\| (\nabla_{u}F)(u_t,v_t) \|^2
				\\
				&&
				\leq
				F(u^*,v_0)
				+
				\frac{1}{t_n}
				\sum_{t=0}^{t_n-1} |F(u^*,v_t)
				- F(u^*,v_0)|\\
				&& \hspace*{2cm} + \frac{\|u^*-u_0\|^2}{2}
				+
				\frac{1}{2 \cdot t_n}
				\sum_{t=0}^{t_n-1}
				\lambda \cdot
				\| (\nabla_{u}F)(u_t,v_t) \|^2.   
			\end{eqnarray*}
			By (\ref{le1eq3}) and the fourth step of the proof we get
			\[
			\frac{1}{t_n}
			\sum_{t=0}^{t_n-1} |F(u^*,v_t)
			- F(u^*,v_0)|
			\leq
			\frac{1}{t_n}
			\sum_{t=0}^{t_n-1}
			D_n \cdot \|u^*\| \cdot \|v_t-v_0\|
			\leq
			D_n \cdot \|u^*\| \cdot \sqrt{2 \cdot F(u_0,v_0)}.
			\]
			And as in the fourth step of the proof we get
			\[
			\sum_{t=0}^{t_n-1}
			\lambda \cdot
			\| (\nabla_{u}F)(u_t,v_t) \|^2
			\leq
			2 \cdot
			\sum_{t=0}^{t_n-1} (F(u_t,v_t)-F(u_{t+1},v_{t+1}))
			\leq 2 \cdot F(u_0,v_0).
			\]
			Summarizing the above results, the proof is complete.
		\end{proof}
		
With the following two results we can show that the assumptions (\ref{le1eq1}) and (\ref{le1eq2}) in the proof of Theorem \ref{th1} are satisfied.

\begin{lemma}\label{le2}
	Let $\sigma: \R \rightarrow \R$ be bounded and differentiable, and assume that
	its derivative is bounded.
	Let $\alpha_n \geq 1$,
	$t_n \geq L_n$,
	$\gamma_n^* \geq 1$, $B_n \geq 1$, $r \geq 2d$,
	\begin{equation}
		\label{le2eq1}
		|w_{1,1,k}^{(L)}| \leq \gamma_n^* \quad \mbox{for } k=1, \dots, K_n,
	\end{equation}
	\begin{equation}
		\label{le2eq2}
		|w_{k,i,j}^{(l)}| \leq B_n
		\quad
		\mbox{for } l=1, \dots, L-1
	\end{equation}
	and
	\begin{equation}
		\label{le2eq3}
		\|\bw-\bv\|_\infty^2 \leq \frac{8t_n}{L_n} \cdot \max\{ F_n(\bv),1 \}.
	\end{equation}
	Then we have for $X_1,\dots, X_n \in [-\alpha_n,\alpha_n]^d$
	\[
	\| (\nabla_\bw F_n)(\bw) \|
	\leq
	\nconst \cdot K_n^{3/2} \cdot B_n^{2L} \cdot (\gamma_n^*)^2 \cdot \alpha_n^{2} \cdot \sqrt{\frac{t_n}{L_n} \cdot \max\{F_n(\bv),1\}}. 
	\]
\end{lemma}

	\begin{proof}
		The proof follows from Lemma 2 in \citet{DrKo22}. For sake of completeness the proof is given in the appendix.
	\end{proof}

\begin{lemma}
	\label{le3}
	Let $\sigma: \R \rightarrow \R$ be bounded and differentiable, and assume that its derivative
	is 
	Lipschitz continuous and bounded.
	Let $\alpha_n \geq 1$, 
	$t_n \geq L_n$,
	$\gamma_n^* \geq 1$, $B_n \geq 1$, $r \geq 2d$ and assume
	\begin{equation}
		\label{le3eq1}
		|\max\{ (\bw_1)_{1,1,k}^{(L)}, (\bw_2)_{1,1,k}^{(L)}\}| \leq \gamma_n^* \quad
		\mbox{for }k=1,
		\dots, K_n,
	\end{equation}
	
	\begin{equation}
		\label{le3eq2}
		|\max\{(\bw_1)_{k,i,j}^{(l)},(\bw_2)_{k,i,j}^{(l)}\}| \leq B_n
		\quad
		\mbox{for } l=1, \dots, L-1
	\end{equation}
	and
	\begin{equation}
		\label{le3eq3}
		\|\bw_2-\bv\|^2 \leq 8 \cdot \frac{t_n}{L_n} \cdot \max\{ F_n(\bv),1 \}.
	\end{equation}
	Then we have for $X_1,\dots, X_n \in [-\alpha_n,\alpha_n]^d$
	\begin{eqnarray*}
		&&
		\| (\nabla_\bw F_n)(\bw_1) - (\nabla_\bw F_n)(\bw_2) \| \\
		&&
		\leq
		\nconst \cdot \max \{\sqrt{F_n(\bv)},1\} \cdot (\gamma_n^*)^{2} \cdot B_n^{3L} \cdot \alpha_n^{3} \cdot K_n^{3/2} \cdot \sqrt{\frac{t_n}{L_n}} \cdot \|\bw_1-\bw_2\|. 
	\end{eqnarray*}
\end{lemma}
	
	\begin{proof}
		The proof follows from Lemma 3 in \citet{DrKo22}. For the sake of completeness the complete proof is given in the appendix.
	\end{proof}
	
	The next auxiliary result uses a metric entropy bound to control the complexity of a set of over-parametrized deep neural networks.
	
	\begin{lemma}
		\label{le4}
	  Let $\alpha \geq 1$, $\beta>0$ and let $A,B,C \geq 1$.
	Let $\sigma:\R \rightarrow \R$ be $k$-times differentiable
	such that all derivatives up to order $k$ are bounded on $\R$.
	Let $\F$
	be the set of all functions $f_{\bw}$ defined by
	(\ref{se2eq1})--(\ref{se2eq3}) where the weight vector $\bw$
	satisfies
	\begin{equation}
		\label{le5eq1}
		\sum_{j=1}^{K_n} |w_{1,1,j}^{(L)}| \leq C,
	\end{equation}
	\begin{equation}
		\label{le4eq2}
		|w_{k,i,j}^{(l)}| \leq B \quad (k \in \{1, \dots, K_n\},
		i,j \in \{1, \dots, r\}, l \in \{1, \dots, L-1\})
	\end{equation}
	and
	\begin{equation}
		\label{le4eq3}
		|w_{k,i,j}^{(0)}| \leq A \quad (k \in \{1, \dots, K_n\},
		i \in \{1, \dots, r\}, j \in \{1, \dots,d\}).
	\end{equation}
	Then we have for any $1 \leq p < \infty$, $0 < \epsilon < \beta$ and
	$x_1^n \in \Rd$
	\begin{eqnarray*}
		&&\Nu_p \left(
		\epsilon, \{ T_\beta f \cdot 1_{[-\alpha,\alpha]^d} \, : \, f \in \F \}, x_1^n
		\right)
		\\
		&&
		\leq \left(\nconst\cdot \frac{\beta^p} {\epsilon^p}\right)^{\nconst\cdot \alpha^d \cdot B^{(L-1)\cdot d} \cdot A^d \cdot \left(\frac{C}{\epsilon}\right)^{d/k}+ \nconst
		}.
		\\
	\end{eqnarray*}
	
	\end{lemma}
	
	\begin{proof}
		See Lemma 4 in \citet{DrKo22}.
	\end{proof}
	
	The next lemma gives a bound on the error of the approximation of a Lipschitz continuous and bounded function by an over-parametrized deep neural network.
	
	\begin{lemma}\label{le6}
		
		Let $\sigma$ be the logistic squasher, let $1\leq \alpha_n \leq \log n$,
		let $m:\Rd \rightarrow \R$ be Lipschitz continuous as well as bounded,
		let $L,r,n \in \N$ with $L \geq 2$, $r \geq 2d$, $n \geq 8d$ and $n \geq\exp(r+1)$
		and let $K, N_n \in \N$ with
		$ 2 \leq K \leq \alpha_n -1$ and
		$N_n \cdot (K^2+1)^{d} \leq K_n$.
		Given $u_1,v_1, \dots, u_{N_n(K^2+1)^{d}},v_{N_n (K^2+1)^{d}} \in [-K-\frac{2}{K},K]^d$, choose $\bw$ such that
		\begin{equation}
			\label{le6eq4}
			w_{k,j,j}^{(0)}=4d \cdot K^2 \cdot (\log n)^2
			\quad \mbox{and} \quad w_{k,j,0}^{(0)}=-4d \cdot K^2 \cdot (\log n)^2 \cdot u_k^{(j)}
			\quad \mbox{for } j \in \{1, \dots, d\},
		\end{equation}
		\begin{equation}
			\label{le6eq5}
			w_{k,j+d,j}^{(0)}=-4d \cdot K^2 \cdot (\log n)^2
			\quad \mbox{and} \quad
			w_{k,j+d,0}^{(0)}=4d \cdot K^2 \cdot (\log n)^2 \cdot v_k^{(j)}
			\quad\mbox{for } j \in \{1, \dots, d\},
		\end{equation}
		\begin{equation}
			\label{le6eq6}
			w_{k,s,t}^{(0)}=0
			\quad \mbox{if }
			s \leq 2d,
			s \neq t, s \neq t+d \mbox{ and } t>0,
		\end{equation}
		\begin{equation}
			\label{le6eq7}
			w_{k,1,t}^{(1)}= 8 \cdot (\log n)^2\quad \mbox{for } t \in \{1, \dots, 2d\},
		\end{equation}
		\begin{equation}
			\label{le6eq8}
			w_{k,1,0}^{(1)} = - 8 \cdot (\log n)^2 \left(2d-\frac{1}{n}\right),
		\end{equation}
		\begin{equation}
			\label{le6beq9}
			w_{k,1,t}^{(1)}=0
			\quad \mbox{for }
			t > 2d,
		\end{equation}
		\begin{equation}
			\label{le6eq10}
			w_{k,1,1}^{(l)}= 6 \cdot (\log n)^2 \quad \mbox{for } l \in \{2, \dots, L\},
		\end{equation}
		\begin{equation}
			\label{le6eq11}
			w_{k,1,0}^{(l)} = -3 \cdot (\log n)^2 \quad \mbox{for } l \in \{2, \dots, L\}
		\end{equation}
		and
		\begin{equation}
			\label{le6eq12}
			w_{k,1,t}^{(l)}=0
			\quad \mbox{for }
			t > 1 \mbox{ and } l \in \{2, \dots, L\}
		\end{equation} 
		for all $k \in \{1, \dots, N_n \cdot (K^2+1)^{d}\}$.
		
		Then there exists  $u_1,v_1, \dots, u_{N_n(K^2+1)^{d}},v_{N_n(K^2+1)^{d}} \in [-K-\frac{2}{K},K]^d$ and
		\begin{equation}
			\label{le6eq13}
			\alpha_1, \dots, \alpha_{N_n(K^2+1)^{d}} \in
			\left[- \frac{\|m\|_\infty}{N_n},
			\frac{\|m\|_\infty}{N_n} \right]
		\end{equation}
		such that for all pairwise distinct $j_1, \dots, j_{N_n(K^2+1)^{d}} \in \{1, \dots, K_n\}$  
		\begin{eqnarray}\label{le6eq1}
			&&\int \left|f_{\bar{\bw}}(x)-m(x)\right|^2 \PROB_X(dx)\nonumber\\
			&&\leq
			\nconst \cdot \left( \frac{1}{K} + \frac{N_n^2\cdot K^{4d}}{n^2}
			+\left(\frac{K^{2d}}{n} +1\right)^2 \cdot \PROB_X(\mathbb{R}^d\setminus [-K,K]^d)\right)
		\end{eqnarray}
		holds
		for all weight vectors $\bar{\bw}$ which satisfy 
		\begin{equation}
			\label{le6beq2}
			\bar{w}_{1,1,j_k}^{(L)}=\alpha_k \ (k \in \{1, \dots, N_n \cdot (K^2+1)^{d}\}), \quad
			\bar{w}_{1,1,k}^{(L)}=0 \ (k \notin \{j_1, \dots, j_{N_n (K^2+1)^{d}}\})
		\end{equation}
		and
		\begin{equation}
			\label{le6eq3}
			|w_{s,k,i}^{(l)}-\bar{w}_{j_s,k,i}^{(l)}| \leq \log n
			\quad \mbox{for all } l \in \{0, \dots, L-1\}, s \in \{1, \dots, N_n \cdot(K^2+1)^{d}\}.
		\end{equation}
		Additionally we get
		\begin{equation}\label{le6eq14}
		\|f_{\bar{\bw}}\|_\infty \leq \nconst \cdot \left(3^d + \frac{(K^2+1)^{d}}{n} \right)
		\end{equation}
	where $\const$ depends on $\|m\|_\infty$.
	\end{lemma}
	
	\begin{proof}
	
	Subdivide the cube $[-K-\frac{2}{K},K]^d$ in $(K^2+1)^{d}$ equidistant cubes $C_i$ of side length $\frac{2}{K}$.
	For simplicity we number these cubes $C_{i}$ by $i \in \{1,\dots,(K^2+1)^{d}\}$, such that $C_{i}$ corresponds to the cube
	\[
	[u_{i}^{(1)},v_{i}^{(1)}) \times \dots \times [u_{i}^{(d)},v_{i}^{(d)}).
	\]
	Let $C_{Lip}$ be the Lipschitz constant of $m$.	
	
	We apply Lemma 6 from \citet{DrKo22} with $a_i = -K-\frac{2}{K}$, $b_i=K$ and $K^2+1$ instead of $K$ to $m/N_n$ and $\delta=\frac{1}{K^2}$. This results in
	\[
	\left| f_{\bar{\bw}}(x)- \frac{1}{N_n} \cdot m(x)
	\right|  \leq
	\nconst \cdot \left(
	\frac{C_{Lip}}{N_n} \cdot
	\frac{2}{K} +  (K^2+1)^{d}\cdot \frac{1}{n} 
	\right)
	\]
	for all $x \in [-K-\frac{2}{K},K]^d$ which are not contained in
	\begin{equation}\label{le7eq2}
		A :=
		\bigcup_{i \in \{0,1,...,K^2+1\}} \bigcup_{j \in \{1,...,d\}} \left\{x \in \mathbb{R}^d:\left|x^{(j)}- \left(
		-K - \frac{2}{K} + i \cdot \frac{2}{K}
		\right)\right| < \delta \right\}. 
	\end{equation}
	We repeat the whole construction $N_n$ many times. Thus we obtain an approximation   $f_{\bar{\bw}}$ of
	\[
	N_n \cdot  \frac{1}{N_n} \cdot m(x)
	\]
	which satisfies
	\begin{equation}
		\label{ple7eq*}
		\left| f_{\bar{\bw}}(x)-  m(x)
		\right|  \leq
		\nconst \cdot \left(
		\frac{1}{K} + N_n \cdot (K^2+1)^{d} \cdot \frac{1}{n} 
		\right)
	\end{equation}
	for $x \notin A$.
	
	Next we shift the grid along the $j$-th component so that $[-K,K]^d$ is always covered. This means we modify all $u_{i}^{(j)},v_{i}^{(j)}$ by the same additional summand which is chosen from the set
	\[
	\left\{
	k \cdot
	\frac{2}{K^2} \quad : \quad k=0,1, \dots, K-1
	\right\}
	\]
	for fixed $j \in \{1,\dots,d\}$. Thus we obtain $K$
	different versions of $f_{\bar{\bw}}$ that still satisfy (\ref{ple7eq*}) for all $x \in [-K,K]^d$ up to corresponding versions of $A$.
	
	Since we shift the grid of cubes we obtain for fixed $j \in \{1,\dots, d\}$ $K$ disjoint versions of $\bigcup_{i \in \{0,1,...,K^2+1\}} \left\{x \in \mathbb{R}^d:\left|x^{(j)}- \left(
	-K - \frac{2}{K} + i \cdot \frac{2}{K}
	\right)\right| < \delta \right\}$.
	The sum of $\PROB_X$-measures of these $K$ disjoint sets is less than or equal to one. Therefore at least one of them must have measure less than or equal to $\frac{1}{K}$. Consequently we can shift $u_i$ and $v_i$ so that
	\begin{equation*}
		\PROB_X \left(A \right)
		\leq \sum_{j \in \{1,\dots,d\}} \frac{1}{K} = \frac{d}{K}
	\end{equation*}
	holds.
	
	Now we have shown that there exists a shifted version of the grid such that the set $A$ has a measure less than or equal to $\frac{d}{K}$.
	By inequality (\ref{ple7eq*}) we get that $|f_{\bar{\bw}}(x)-m(x)| \leq \const \cdot \left(\ \frac{1}{K} + \frac{N_n \cdot K^{2d}}{n} \right)$ holds for $x \in [-K,K]^d\setminus A$. 
	
	From the second assertion of Lemma 6 in \citet{DrKo22} we obtain
		\[
		|f_{\bar{\bw}}(x)| \leq \|m\|_\infty\cdot \left( 3^d + (K^2+1)^{d} \cdot \frac{1}{n} \right)
		\]
		for $x \in \mathbb{R}^d$.
	
	Summarizing the above results we get
	\begin{eqnarray*}
		&&\int \left|f_{\bar{\bw}}(x)-m(x)\right|^2 \PROB_X(dx)\\
		&&=\int_{[-K,K]^d\setminus A}\left|f_{\bar{\bw}}(x)-m(x)\right|^2 \PROB_X(dx) +
		\int_{A}\left|f_{\bar{\bw}}(x)-m(x)\right|^2 \PROB_X(dx)\\
		&& \qquad + \int_{\mathbb{R}^d\setminus {[-K,K]^d}} \left|f_{\bar{\bw}}(x)-m(x)\right|^2 \PROB_X(dx)\\
		&&\leq \const^2 \left(\frac{1}{K} + \frac{N_n\cdot K^{2d}}{n}\right)^2  + \nconst
		\left(3^d + \frac{K^{2d}}{n}\right)^2  \cdot \frac{d}{K}
		\\
		&& \quad
		+ \nconst \left(3^d +\frac{ K^{2d}}{n} \right)^2
		\cdot \PROB_X(\mathbb{R}^d\setminus [-K,K]^d),
	\end{eqnarray*}
	which implies the assertion.

	\end{proof}

	\subsection{Proof of Theorem \ref{th1}}
	
	\begin{proof}
		Let $\epsilon >0$, $K\in \N$ arbitrary and $N_n =\lceil \nconst \cdot (\log n)^{18L}\rceil$.
		Furthermore, let  $\bar{m}:\Rd \rightarrow \R$ be a Lipschitz continuous and bounded function such that
		\begin{equation}
			\label{pth1eq0}
			\int | \bar{m}(x)-m(x)|^2 \PROB_X (dx) \leq \epsilon.
		\end{equation}
		
		We denote by $A_n$ the event that firstly there exists pairwise disjoint $j_1,\dots,j_{N_n  (K^2+1)^d}$ such that  the weight vector $\bw^{(0)}$
		satisfies
		\[
		| (\bw^{(0)})_{j_s,k,i}^{(l)}-\bw_{s,k,i}^{(l)}| \leq \log n
		\quad \mbox{for all } l \in \{0, \dots, L-1\},
		s \in \{1, \dots, N_n\cdot (K^2+1)^d \}
		\]
		for some weight vector $\bw$ which satisfies
		the conditions
		(\ref{le6eq4})--(\ref{le6eq12})
		of Lemma \ref{le6} for $\bar{m}$
		and that secondly the inequality
		\[
		\frac{1}{n}\sum_{i=1}^{n}Y_i^2 \leq \beta_n^3
		\]
		holds.
		
		In case that $A_n$ holds $\alpha_k$ is chosen as in Lemma \ref{le6} for $\bar{m}$, otherwise we set $\alpha_1=...=\alpha_{N_n(K^2+1)^d }=0$.
		Then we define the weight vectors
		$\bw^*$ for given $\bw$ by
		\[
		(\bw^*)_{k,i,j}^{(l)} = \bw_{k,i,j}^{(l)} \quad
		\mbox{for all } l=0,\dots, L-1,
		\]
		\[
		(\bw^*)_{1,1,j_k}^{(L)} = \alpha_k \quad \mbox{for all } k=1,\dots, N_n\cdot (K^2+1)^d, 
		\]
		\[
		(\bw^*)_{1,1,k}^{(L)} = 0 \quad \mbox{for all } k \notin \{j_1,\dots,
		j_{N_n(K^2+1)^d}\}
		\]
		and $(\bw^*)^{(0)}$ by
		\[
		((\bw^*)^{(0)})^{(l)}_{k,i,j} = (\bw^{(0)})^{(l)}_{k,i,j} \quad
		\mbox{for all } l=0,\dots, L-1,
		\]
		\[
		((\bw^*)^{(0)})_{1,1,j_k}^{(L)} = \alpha_k \quad \mbox{for all } k=1,\dots, N_n\cdot (K^2+1)^d, 
		\]
		\[
		((\bw^*)^{(0)})_{1,1,k}^{(L)} = 0 \quad \mbox{for all } k \notin \{j_1,\dots,
		j_{N_n(K^2+1)^d}\}.
		\]

		In the {\it first step of the proof} we start by
		decomposing the  $L_2$ error of $m_n$ in a sum of several terms.
		We have
		\begin{eqnarray*}
			&&
			\int | m_{n}(x)-m(x)|^2 \PROB_X (dx)
			\\
			&&
			=(\EXP\{ |m_{n}(X)-Y|^2 | \D_n\} - \EXP\{ |m(X)-Y|^2 \}) \cdot 1_{A_n}\\
			&& \hspace*{0.3cm}+\int | m_{n}(x)-m(x)|^2 \PROB_X (dx) \cdot 1_{A_n^c}\\
			&&
			= \left( \EXP\{ |m_{n}(X)-Y|^2 | \D_n\} 	- (1+\epsilon) \cdot
			\EXP\{ |m_{n}(X)-T_{\beta_n} Y|^2 | \D_n\}\right)\cdot 1_{A_n}\\
			&&
			\hspace*{0.3cm}
			+ \left((1+\epsilon) \cdot \EXP\{ |m_{n}(X)-T_{\beta_n} Y|^2 | \D_n\} - (1+\epsilon) \cdot \frac{1}{n} \sum_{i=1}^n |m_{n}(X_i)- T_{\beta_n} Y_i|^2 \right)\cdot 1_{A_n}\\
			&&
			\hspace*{0.3cm}
			+
			\left((1+\epsilon) \cdot \frac{1}{n} \sum_{i=1}^n |m_{n}(X_i)- T_{\beta_n} Y_i|^2 - (1+\epsilon) \cdot \frac{1}{n} \sum_{i=1}^n |f_{\bw^{(t_n)}}(X_i) - T_{\beta_n} Y_i|^2\right) \cdot 1_{A_n}\\
			&&
			\hspace*{0.3cm}+\left((1+\epsilon) \cdot \frac{1}{n} \sum_{i=1}^n |f_{\bw^{(t_n)}}(X_i) - T_{\beta_n} Y_i|^2	-
			(1+\epsilon)^2 \cdot \frac{1}{n} \sum_{i=1}^n |f_{\bw^{(t_n)}}(X_i) - Y_i|^2)\right)\cdot 1_{A_n}
			\\
			&&
			\hspace*{0.3cm}
			+
			\left((1+\epsilon)^2 \cdot \frac{1}{n} \sum_{i=1}^n |f_{\bw^{(t_n)}}(X_i) - Y_i|^2 - \EXP\{ |m(X)-Y|^2 \}\right) \cdot 1_{A_n}\\
			&&
			\hspace*{0.5cm} + \int | m_{n}(x)-m(x)|^2 \PROB_X (dx) \cdot 1_{A_n^c}  \\
			&&
			= \sum_{j=1}^6 T_{j,n}.
		\end{eqnarray*}

		In the {\it second step of the proof} we show
		\[
		\limsup_{n \rightarrow \infty} \EXP T_{j,n} \leq 0 \quad
		\mbox{for } j \in \{1,4\}.
		\]
		By using $(a+b)^2 \leq (1+\epsilon) \cdot a^2 + (1 + \frac{1}{\epsilon}) \cdot b^2$ for $a,b \in \R$ we obtain
		\[
		\EXP T_{1,n} \leq \left(1 + \frac{1}{\epsilon} \right) \cdot \EXP\{|T_{\beta_n}Y-Y|^2\}
		\]
		and
		\[
		\EXP T_{4,n} \leq (1+\epsilon) \cdot \left(1 + \frac{1}{\epsilon} \right) \cdot \EXP\{|T_{\beta_n}Y-Y|^2\}.
		\]
		Together with $\beta_n \rightarrow \infty$ $(n \rightarrow \infty)$ and
		$\EXP Y^2 < \infty$ this implies the assertion of the second step.
		
		In the {\it third step of the proof} we show
		\[
		\limsup_{n \rightarrow \infty} \EXP T_{3,n} \leq 0. 
		\]
		If $|y| \leq \beta_n$ then it holds 
		\[
		| T_{\beta_n} z - y| \leq |z-y|
		\]
		for
		any $z \in \R$.
		This implies
		\begin{align*}
			\frac{1}{n} \sum_{i=1}^n |m_{n}(X_i)- T_{\beta_n} Y_i|^2 &=
			\frac{1}{n} \sum_{i=1}^n |T_{\beta_n}f_{\bw^{(t_n)}}(X_i) - T_{\beta_n} Y_i|^2     
			\\
			&\leq
			\frac{1}{n} \sum_{i=1}^n |f_{\bw^{(t_n)}}(X_i) - T_{\beta_n} Y_i|^2.
		\end{align*}
		Thus the assertion of the third step holds.
		
		In the {\it fourth step of the proof} we show that the assumptions (\ref{le1eq1}) - (\ref{le1eq3}) of Lemma \ref{le1} are satisfied which means that
		\begin{equation*}
			\| (\nabla_{\bw} F)(\bw) \| \leq L_n
		\end{equation*}
		for all $\bw \in S:=\left\{
		\bv \, : \,
		\|\bv - \bw^{(0)}\| \leq
		2 \cdot \sqrt{F(\bw^{(0)})} + 1
		\right\}$,
		\begin{equation*}
			\| (\nabla_{\bw} F)(\bw) - (\nabla_{\bw} F)(\tilde{\bw}) \| \leq
			L_n \cdot \|\bw-\tilde{\bw}\|
		\end{equation*}
		for all $\bw, \tilde{\bw} \in S$ 
		and
		\begin{align*}
			&|F(\bw^*)-F((\bw^*)^{(0)})|\\
			&\leq D_n \cdot  \|((\bw^*)^{(L)}_{1,1,k})_{k=1,\dots,K_n}\| \cdot \|(\bw_{i,j,k}^{(l)})_{i,j,k,l:l<L}-((\bw^{(0)})_{i,j,k}^{(l)})_{i,j,k,l:l<L}\|
		\end{align*}
		for all 
		\begin{align*}
			&(\bw_{i,j,k}^{(l)})_{i,j,k,l:l<L} \in \tilde{S}\\
			&\hspace*{1.4cm}:= \left\{(\tilde{\bw}_{i,j,k}^{(l)})_{i,j,k,l:l<L} :  \|(\tilde{\bw}_{i,j,k}^{(l)})_{i,j,k,l:l<L}-((\bw^{(0)})_{i,j,k}^{(l)})_{i,j,k,l:l<L}\| \leq \sqrt{2 \cdot F(\bw^{(0)})} \right\}
		\end{align*} hold, 
		if $A_n$ holds.
		
		If $A_n$ holds, then we have
		\[
		F_n(\bw^{(0)})
		=
		\frac{1}{n} \sum_{i=1}^n Y_i^2 \leq \beta_n^3.
		\]
		
		Let $\bw \in S$. Then we can conclude that 
		\begin{align*}
			\|(\bw_{i,j,k}^{(l)})_{i,j,k,l:1\leq l<L}\|_\infty &\leq \|\bw-\bw^{(0)}\| +  \|((\bw^{(0)})_{i,j,k}^{(l)})_{i,j,k,l:1\leq l<L}\|_\infty\\
			&\leq2 \cdot  \sqrt{ F_n(\bw^{(0)})} + 1+  \nconst \cdot (\log n)^2\\
			&\leq \nconst \cdot (\log n)^2
		\end{align*}
		and 
		\begin{align*}
			\|(\bw^{(L)}_{1,1,k})_{k=1,\dots,K_n}\|_\infty &\leq \|\bw-\bw^{(0)}\| +  \|((\bw^{(0)})^{(L)}_{1,1,k})_{k=1,\dots,K_n}\|_\infty\\
			&\leq 2 \cdot  \sqrt{ F_n(\bw^{(0)})} + 1\\
			&\leq \nconst \cdot (\log n)^{3/2}.
		\end{align*}
		Hence (\ref{le2eq1})-(\ref{le3eq3}) are satisfied for $B_n = \mconst \cdot (\log n)^2$ and $\gamma_n^*=\const \cdot (\log n)^{3/2}$.
		By Lemma \ref{le2} and Lemma \ref{le3} we get for $\alpha_n=\nconst$ that (\ref{le1eq1}) and (\ref{le1eq2}) are satisfied provided that $L_n \geq K_n^{3/2} \cdot (\log n)^{6L+5}$.
		
		It remains to show that (\ref{le1eq3}) holds. Let $\bw$ such that 
		$ (\bw_{i,j,k}^{(l)})_{i,j,k,l:l<L} \in \tilde{S}.$
		Then we have
		\begin{align*}
			&|F_n(\bw^*)-F_n((\bw^*)^{(0)})|\\
			&= \Bigg| \frac{1}{n} \sum_{i=1}^{n} |f_{\bw^*}(X_i)-Y_i |^2 - \frac{1}{n} \sum_{i=1}^{n} |f_{(\bw^*)^{(0)}}(X_i)-Y_i |^2  \Bigg| \\
			& =  \frac{1}{n} \sum_{i=1}^{n} \left(f_{\bw^*}(X_i)-Y_i + f_{(\bw^*)^{(0)}}(X_i)-Y_i\right) \left(f_{\bw^*}(X_i) - f_{(\bw^*)^{(0)}}(X_i)\right)\\ 
			& \leq  \Bigg(\frac{1}{n} \sum_{i=1}^{n} \left(f_{\bw^*}(X_i)-Y_i + f_{(\bw^*)^{(0)}}(X_i)-Y_i\right)^2 \Bigg)^{1/2}\\
			&\hspace*{2cm}
			\Bigg(\frac{1}{n} \sum_{i=1}^{n}	\left(f_{\bw^*}(X_i) - f_{(\bw^*)^{(0)}}(X_i)\right)^2\Bigg)^{1/2}\\
			& \leq \Bigg(\frac{2}{n} \sum_{i=1}^{n}  \left(f_{\bw^*}(X_i)+f_{(\bw^*)^{(0)}}(X_i)\right)^2 + \frac{8}{n} \sum_{i=1}^{n} Y_i^2\Bigg)^{1/2}\\
			& \hspace*{2cm} \Bigg(\frac{1}{n} \sum_{i=1}^{n}	\left(\sum_{k=1}^{K_n}\left|(\bw^*)^{(L)}_{1,1,k}\right|^2\cdot \sum_{k=1}^{K_n}\left|f_{\bw^*,{k,1}}^{(L)}(X_i) - f_{(\bw^*)^{(0)},{k,1}}^{(L)}(X_i)\right|^2\right)\Bigg)^{1/2}.
		\end{align*}
		
		For the first term we get
		\begin{align*}
			&\Bigg( \frac{2}{n} \sum_{i=1}^{n}  \left(f_{\bw^*}(X_i)+f_{(\bw^*)^{(0)}}(X_i)\right)^2 + \frac{8}{n} \sum_{i=1}^{n} Y_i^2\Bigg)^{1/2}\\
			&=\Bigg( \frac{2}{n} \sum_{i=1}^{n}  \left(f_{(\bw^*)}(X_i)-f_{(\bw^*)^{(0)}}(X_i)+2\cdot f_{(\bw^*)^{(0)}}(X_i)\right)^2 + \frac{8}{n} \sum_{i=1}^{n} Y_i^2\Bigg)^{1/2}\\
			&\leq \Bigg(\frac{2}{n} \sum_{i=1}^{n} 2 \cdot \left(f_{\bw^*}(X_i)-f_{(\bw^*)^{(0)}}(X_i)\right)^2+ \frac{2}{n} \sum_{i=1}^{n} 8\cdot f_{(\bw^*)^{(0)}}(X_i)^2 + \frac{8}{n} \sum_{i=1}^{n} Y_i^2\Bigg)^{1/2}\\
			&\leq \Bigg(4 \cdot \frac{1}{n} \sum_{i=1}^{n}	\left(\sum_{k=1}^{K_n}\left|(\bw^*)^{(L)}_{1,1,k}\right|^2\cdot \sum_{k=1}^{K_n}\left|f^{(L)}_{\bw^*,k,1}(X_i) - f^{(L)}_{(\bw^*)^{(0)},k,1}(X_i)\right|^2\right)\\
			&\hspace*{2cm}+ 16 \cdot \frac{1}{n} \sum_{i=1}^{n} f_{(\bw^*)^{(0)}}(X_i)^2 + \frac{8}{n} \sum_{i=1}^{n} Y_i^2\Bigg)^{1/2}\\
			&\leq \Bigg( 4  \cdot \sum_{k=1}^{K_n}\left|(\bw^*)_{1,1,k}^{(L)}\right|^2\cdot \max_{i=1,\dots,n} \sum_{k=1}^{K_n}\left|f^{(L)}_{\bw^*,k,1}(X_i) - f^{(L)}_{(\bw^*)^{(0)},k,1}(X_i)\right|^2\\
			&\hspace*{2cm}+ 16 \cdot \frac{1}{n} \sum_{i=1}^{n}  f_{(\bw^*)^{(0)}}(X_i)^2 + \frac{8}{n} \sum_{i=1}^{n} Y_i^2\Bigg)^{1/2}.
		\end{align*}
		By Lemma \ref{le6} for $\bar{m}$ we get $|\alpha_k|\leq \frac{\nconst}{N_n}$ and
		\begin{align*}
			f_{(\bw^*)^{(0)}}(X_i) \leq \nconst \cdot \left(3^{d} + \frac{(K^2+1)^d}{n}\right).
		\end{align*}
		From the proof of Lemma \ref{le2} we know that
		\begin{align}\label{proofth1eq1}
			&|f_{\bw^*,k,1}^{(L)}(x) - f_{(\bw^*)^{(0)},k,1}^{(L)}(x)|\nonumber \\
			&\hspace*{3cm}\leq \nconst  \cdot (\log n)^{2L} \cdot \max_{i,j,s:s<L}|(\bw^*)_{k,i,j}^{(s)}-((\bw^*)^{(0)})_{k,i,j}^{(s)}|
		\end{align}
		holds. Therefore we obtain for $(\bw_{i,j,k}^{(l)})_{i,j,k,l:l<L} \in \tilde{S}$
		\begin{align*}
			&\max_{i=1,\dots,n} \sum_{k=1}^{K_n}\left|f^{(L)}_{\bw^*,k,1}(X_i) - f^{(L)}_{(\bw^*)^{(0)},k,1}(X_i)\right|^2\\
			&\qquad \leq  \nconst  \cdot (\log n)^{4L} \cdot \sum_{k=1}^{K_n} \max_{i,j,s:s<L}|(\bw^*)_{k,i,j}^{(s)}-((\bw^*)^{(0)})_{k,i,j}^{(s)}|^2\\
			&\qquad \leq  \const  \cdot (\log n)^{4L} \cdot \|(\bw_{i,j,k}^{(l)})_{i,j,k,l;l<L}-((\bw^{(0)})_{i,j,k}^{(l)})_{i,j,k,l;l<L}\|^2\\
			&\qquad \leq \const \cdot (\log n)^{4L+3}.
		\end{align*}
		From this together with the definition of $\bw^*$ we can conclude
		\begin{align*}
			&\Bigg(\frac{2}{n} \sum_{i=1}^{n}  \left(f_{\bw^*}(X_i)+f_{(\bw^*)^{(0)}}(X_i)\right)^2 + \frac{8}{n} \sum_{i=1}^{n} Y_i^2\Bigg)^{1/2}\\
			&\leq \Bigg(4 \cdot \sum_{k=1}^{K_n}\left|(\bw^*)_{1,1,k}^{(L)}\right|^2\cdot \const \cdot (\log n)^{4L+3} +16 \cdot \frac{1}{n} \sum_{i=1}^{n} f_{(\bw^*)^{(0)}}(X_i)^2 + \frac{8}{n} \sum_{i=1}^{n} Y_i^2\Bigg)^{1/2}\\
			&\leq \Bigg(\nconst \cdot \frac{N_n \cdot (K^2+1)^d}{N_n^2} \cdot (\log n)^{4L+3}+ \left( \nconst \cdot \left(3^d + \frac{ (K^2+1)^d}{n}\right) \right)^2  + 8\cdot c_1^3 \cdot (\log n)^3\Bigg)^{1/2}\\
			& \leq \nconst \Bigg(\frac{(K^2+1)^d}{N_n}\cdot (\log n)^{4L+3} + \left(\frac{ (K^2+1)^d}{n}\right)^2 + (\log n)^3\Bigg)^{1/2}.
		\end{align*}
		Furthermore, due to (\ref{proofth1eq1}) we have
		\begin{align*}
			&\Bigg(\frac{1}{n} \sum_{i=1}^{n}	\left(\sum_{k=1}^{K_n}\left|(\bw^*)_{1,1,k}^{(L)}\right|^2\cdot \sum_{k=1}^{K_n}\left|f^{(L)}_{\bw^*,k,1}(X_i) - f^{(L)}_{(\bw^*)^{(0)},k,1}(X_i)\right|^2\right)\Bigg)^{1/2}\\
			&\leq\left(\sum_{k=1}^{K_n}\left|(\bw^*)_{1,1,k}^{(L)}\right|^2\cdot \sum_{k=1}^{K_n}\max_{i=1,\dots,n}\left|f^{(L)}_{\bw^*,k,1}(X_i) - f^{(L)}_{(\bw^*)^{(0)},k,1}(X_i)\right|^2\right)^{1/2}\\
			&\leq \|((\bw^*)^{(L)}_{1,1,k})_{k=1,\dots,K_n}\| \cdot \nconst \cdot (\log n)^{2L} \cdot \| (\bw^{(l)}_{i,j,k})_{i,j,k,l:l<L} - ((\bw^{(0)})_{i,j,k}^{(l)})_{i,j,k,l:l<L}\|.
		\end{align*}
		This yields
		\begin{align*}
			&|F_n((\bw^*)^{(t)})-F_n((\bw^*)^{(0)})|\\ 
			&\leq \nconst\Bigg(\frac{(K^2+1)^{d/2}}{N_n^{1/2}}\cdot (\log n)^{4L+3/2} + \frac{ (K^2+1)^d\cdot (\log n)^{2L}}{n} + (\log n)^{2L+3/2}\Bigg)\\
			&\hspace*{2cm} \cdot \|((\bw^*)^{(L)}_{1,1,k})_{k=1,\dots,K_n}\| \cdot\| (\bw^{(l)}_{i,j,k})_{i,j,k,l:l<L} - ((\bw^{(0)})_{i,j,k}^{(l)})_{i,j,k,l:l<L}\| .
		\end{align*}
		Thus (\ref{le1eq3}) is satisfied with 
		\begin{align*}
			D_n &= \const \cdot (\log n)^{2L} \cdot \Bigg(\frac{(K^2+1)^{d/2}}{N_n^{1/2}}\cdot (\log n)^{2L+3/2} + \frac{(K^2+1)^d}{n} + (\log n)^{3/2}\Bigg).
		\end{align*}
		For $n$ large we get
		\begin{align*}
			D_n \leq \nconst \cdot (\log n)^{4L+2}.
		\end{align*}
		
		In the {\it fifth step of the proof} we show
		\[
		\PROB(A_n^c) \leq \frac{\nconst}{(\log n)^3}.
		\]
		
		To show this, we first bound the probability that the weight vector $\bw^{(0)}$ does not satisfy the first condition of the event $A_n$. For this, we consider a sequential choice of weights in the $K_n$ fully connected neural networks. Each of these $K_n$ fully connected neural networks contains $(r+1)+(L-2)\cdot r \cdot (r+1) + r \cdot (d+1)$ weights. Therefore, the probability that all these weights never satisfy condition (\ref{le6eq3}) for $s=1$ is bounded from below by
		\[
		\left(
		\frac{\log n}{40 d \cdot (\log n)^2}
		\right)^{(r+1)+(L-2) \cdot r \cdot (r+1)}
		\cdot
		\left(
		\frac{ \log n}{16d\cdot (\log n)^2 \cdot n^{\tau}}
		\right)^{r \cdot (d+1)}.
		\]
		Consequently, the probability that this condition is never satisfied in the first\\ $n^{r(d+1)\tau+1}$ many fully connected neural networks for $j_1$ is for large $n$ bounded from above by
		\begin{align*}
			&\left(
			1 - \left(
			\frac{1}{40 d \cdot \log n}
			\right)^{(r+1)+(L-2) \cdot r \cdot (r+1)}
			\cdot
			\left(
			\frac{1}{16d \cdot \log n \cdot n^{\tau}}
			\right)^{r \cdot (d+1)}
			\right)^{n^{r(d+1)\tau+1}}\\
			& \hspace{2cm}\leq \left(1-n^{-r(d+1)\tau-0,5}\right)^{n^{r(d+1)\tau+1}}.
		\end{align*}
		
		Because of condition (\ref{th1eq2}) we have $K_n \geq N_n\cdot (K^2+1)^d \cdot n^{r(d+1)\tau+1}$ for large $n$.
		This implies that for large $n$ condition (\ref{le6eq3}) is satisfied outside of
		an event of probability
		\begin{align*}
			&N_n\cdot (K^2+1)^d \cdot
			\left(1-n^{-r(d+1)\tau-0.5}\right)^{n^{r(d+1)\tau+1}}\\
			& \leq N_n\cdot (K^2+1)^d \cdot  \left(\exp\left(-n^{-r(d+1)\tau-0.5}\right)\right)^{n^{r(d+1)\tau+1}}\\
			& \leq N_n\cdot (K^2+1)^d \cdot \exp\left(-n^{0.5}\right)\\
			&\leq n^\kappa \cdot \exp\left(-n^{0.5}\right)\\
			&\leq \frac{\nconst}{n}.
		\end{align*}
		
		Then we obtain for large $n$ by Markov's inequality 
		\begin{align*}
			\PROB(A^c) &\leq \frac{\const}{n} + \PROB\left\{\frac{1}{n}\sum_{i=1}^{n} Y_i^2 > \beta_n^3\right\}\\
			&\leq \frac{\const}{n} + \frac{\EXP\{\frac{1}{n}\sum_{i=1}^{n} Y_i^2 \}}{\beta_n^3}\\
			& \leq \frac{\const}{n} + \frac{\EXP\{Y^2 \}}{\beta_n^3}\\
			& \leq \frac{\nconst}{(\log n)^3}
		\end{align*}
		where the last inequality holds since $\EXP Y^2 < \infty$.
		
		In the {\it sixth step of the proof} we show
		\[
		\limsup_{n \rightarrow \infty} \EXP T_{2,n} \leq 0. 
		\]
		We have
		\begin{align*}
			\frac{1}{1+\epsilon}\cdot\EXP \left\{T_{2,n} \right\}
			&\leq \int_0^{4 \cdot \beta_n^2}\PROB \Bigg\{\Big(\EXP\{ |m_{n}(X)-T_{\beta_n} Y|^2 | \D_n\}\\
			&\hspace*{2cm} - \frac{1}{n} \sum_{i=1}^n |m_{n}(X_i)- T_{\beta_n} Y_i|^2\Big) \cdot 1_{A_n} >t\Bigg\} \, dt\\
			&\leq n^{\frac{-1}{4(d+2)}} + \int_{n^{\frac{-1}{4(d+2)}}}^{4 \cdot \beta_n^2} \PROB \Bigg\{\Big(\EXP\{ |m_{n}(X)-T_{\beta_n} Y|^2 | \D_n\}\\
			&\hspace*{2cm}-\frac{1}{n} \sum_{i=1}^n |m_{n}(X_i)- T_{\beta_n} Y_i|^2 \Big) \cdot 1_{A_n} >t \Bigg\} \, dt.
		\end{align*}
		
		W.l.o.g. we can assume that $A_n$ holds. Hence, by Lemma \ref{le1}, it follows that
		\begin{align*}
			\|((\bw^{(t_n)})^{(l)}_{i,j,k})_{i,j,k,l:l<L} - ((\bw^{(0)})_{i,j,k}^{(l)})_{i,j,k,l:l<L}\| \leq \nconst \cdot (\log n)^{3/2}
		\end{align*}
		and				
		\begin{align*}
			\|((\bw^{(t_n)})^{(L)}_{1,1,k})_{k=1,\dots,K_n}\| \leq \nconst \cdot (\log n)^{3/2}.
		\end{align*}
		
		hold. Consequently, we obtain
		\begin{align*}
			&\|((\bw^{(t_n)})^{(l)}_{i,j,k})_{i,j,k,l:1 \leq l<L}\|_\infty\\
			&\leq \|((\bw^{(t_n)})^{(l)}_{i,j,k})_{i,j,k,l:1\leq l<L}-((\bw^{(0)})^{(l)}_{i,j,k})_{i,j,k,l:1 \leq l<L}\|_\infty + \|((\bw^{(0)})^{(l)}_{i,j,k})_{i,j,k,l:1\leq l<L}\|_\infty\\
			&\leq \mmconst{1} \cdot (\log n)^{3/2} + \nconst \cdot (\log n)^2
		\end{align*}
		and
		\begin{align*}
			&\|((\bw^{(t_n)})^{(0)}_{i,j,k})_{i,j,k}\|_\infty\\
			&\leq \|((\bw^{(t_n)})^{(0)}_{i,j,k})_{i,j,k}-((\bw^{(0)})^{(0)}_{i,j,k})_{i,j,k}\|_\infty + \|((\bw^{(0)})^{(0)}_{i,j,k})_{i,j,k}\|_\infty\\
			&\leq \mconst \cdot (\log n)^{3/2} +  \nconst \cdot (\log n)^2 \cdot n^\tau.
		\end{align*}
		This implies that
		$m_{n}$ is contained in the function space
		\[
		\{ T_{\beta_n} f \, : \, f \in \F \}
		\]
		where 
		$\F$ is defined as in Lemma \ref{le4}
		with $C = \nconst \cdot K_n \cdot (\log n)^{3/2}$, $B=\nconst\cdot (\log n)^2$
		and $A = \nconst \cdot (\log n)^2 \cdot  n^{\tau}$. Thus, with Lemma \ref{le4} and standard bounds of empirical process theory (cf., Theorem 9.1 in \citet{GyKoKrWa02}), it follows 
		\begin{eqnarray*}
			&&
			\PROB \Bigg\{
			(\EXP\{ |m_{n}(X)-T_{\beta_n} Y|^2 (X)| \D_n\}
			-
			\frac{1}{n} \sum_{i=1}^n |m_{n}(X_i)- T_{\beta_n} Y_i|^2) \cdot 1_{A_n}
			>t
			\Bigg\}
			\\
			&&
			\leq
			8 \cdot \left(
			\nconst\cdot \frac{\beta_n}{t/8}
			\right)^{
				\nconst \cdot (\log n)^{\nconst} \cdot n^{\tau \cdot d} \cdot \left(
				\frac{\nconst \cdot K_n \cdot (\log n)^{3/2}}{t/8}
				\right)^{d/k} + \nconst
			}
			\cdot
			\exp \left(
			- \frac{n \cdot t^2}{128 \cdot \beta_n^4}
			\right).\\
		\end{eqnarray*}
		Using (\ref{th1eq1}) and $\tau=\frac{1}{d+1}$ we get for $k>(\kappa+1) \cdot d \cdot (d+1) \cdot (d+2)$ and $t>n^{\frac{-1}{4(d+2)}}$ that the left hand side above is for $n$ large enough less than or equal to
		\begin{align*}
			&8 \cdot \left(
			\mmconst{4} \cdot \frac{\beta_n}{t/8}
			\right)^{
				\mmconst{3} \cdot (\log n)^{\mmconst{2}} \cdot  n^{\tau \cdot d} \cdot \left(
				\frac{\mconst\cdot K_n\cdot (\log n)^{3/2}}{t/8}
				\right)^{d/k} + \const}\\
			&\hspace*{2cm}
			\cdot
			\exp \left(
			- \frac{n \cdot t^2}{256 \cdot \beta_n^4}
			\right)
			\cdot
			\exp \left(
			- \frac{n \cdot t^2}{256 \cdot \beta_n^4}
			\right)\\
			&\leq  \exp\left(\nconst \cdot (\log n)^{\mmconst{3}} \cdot n^{\tau \cdot d + (\kappa+1)\cdot \frac{d}{k}}  \cdot  \log\left( \mmconst{5} \cdot \frac{\beta_n}{t/8}\right)\right)\\
			&\hspace*{2cm}\cdot
			\exp \left(
			- \frac{n \cdot t^2}{256 \cdot \beta_n^4}
			\right)
			\cdot
			\exp \left(
			- \frac{n \cdot t^2}{256 \cdot \beta_n^4}
			\right)\\
			&\leq \exp\left( \nconst \cdot \left((\log n)^{\nconst} \cdot n^{\frac{d}{d+1}+(\kappa +1)\frac{d}{k}} - \frac{n^\frac{2d+3}{2(d+2)}}{(\log n)^4}\right)\right) 	\cdot
			\exp \left(
			- \frac{n \cdot t^2}{256 \cdot \beta_n^4}
			\right) \\
			&\leq \exp\left( \mconst \cdot \left((\log n)^{\const} \cdot n^{\frac{d+1}{d+2}} - \frac{n^\frac{2d+3}{2(d+2)}}{(\log n)^4}\right)\right)	\cdot
			\exp \left(
			- \frac{n \cdot t^2}{256 \cdot \beta_n^4}
			\right)\\
			&= \exp\left( \mconst \cdot n^{\frac{d+1}{d+2}} \cdot \left((\log n)^{\mconst} - \frac{n^\frac{1}{2(d+2)}}{(\log n)^4}\right)\right)	\cdot
			\exp \left(
			- \frac{n \cdot t^2}{256 \cdot \beta_n^4}
			\right)\\
			&\leq 
			\nconst
			\cdot
			\exp \left(
			- \frac{n \cdot t^2}{256 \cdot \beta_n^4}
			\right)\\
			&\leq
			\nconst
			\cdot
			\exp \left(
			- \frac{n^{\frac{2d+3}{2(d+2)}}}{256 \cdot \beta_n^4}
			\right)
		\end{align*}
		holds.
		Therefore, we obtain
		\[
		\EXP \left\{
		T_{2,n} 
		\right\}
		\leq
		(1+\epsilon) \cdot \left(n^{\frac{-1}{4(d+2)}}+ 4 \beta_n^2 \cdot
		\nconst
		\cdot
		\exp \left(
		-  \frac{n^{\frac{2d+3}{2(d+2)}}}{256 \cdot \beta_n^4}
		\right)
		\right)
		\rightarrow 0
		\quad (n \rightarrow \infty).
		\]
		
		In the {\it seventh step of the proof} we show
		\[
		\limsup_{n \rightarrow \infty} \EXP\{ T_{6,n} \} \leq 0.
		\]
		Due to the assertion of the fifth step together with the integrability of $m$ we get
		\begin{eqnarray*}
			\EXP\{ T_{6,n} \} &\leq&\left(2 \cdot \int |m_n(x)|^2 \ \PROB_X (dx)  + 2 \cdot \int |m(x)|^2 \ \PROB_X (dx)\right) \cdot \PROB({A_n^c})
			\\
			&\leq&
			2\cdot (\beta_n^2 + \nconst ) \cdot \frac{c_{36}}{(\log n)^3}.
		\end{eqnarray*}
		This implies the assertion of the seventh step.

		In the {\it eighth step of the proof} we bound
		\[
		\EXP T_{5,n}.
		\]
		
		If $A_n$ holds, then as shown in step four, we can apply Lemma \ref{le1} with  
		
		\begin{align*}
			D_n \leq c_{35} \cdot (\log n)^{4L+2}.
		\end{align*}
		This together with the definition of $\bw^*$ and $(\bw^*)^{(0)}$ yields
		\begin{eqnarray*}
			&&
			\frac{1}{n} \sum_{i=1}^n |f_{\bw^{(t_n)}}(X_i) - Y_i|^2
			=
			F_n\left(\bw^{(t_n)}\right)
			\\
			&&
			\leq 	F_n\left((\bw^*)^{(0)}\right)+ D_n\cdot\|((\bw^*)_{1,1,k}^{(L)})_{k=1,\dots,K_n}\| \cdot \sqrt{2 \cdot F_n\left(\bw^{(0)}\right)}\\
			&&\hspace*{1cm}
			+\frac{\|((\bw^*)_{1,1,k}^{(L)})_{k=1,\dots,K_n}-((\bw^{(0)})_{1,1,k}^{(L)})_{k=1,\dots,K_n}\|^2}{2}+ \frac{F_n\left(\bw^{(0)}\right)}{t_n}\\
			&&\leq 
			F_n\left((\bw^*)^{(0)}\right) + \nconst \cdot (\log n)^{4L+2} \cdot \frac{N_ n^{1/2}\cdot (K^2+1)^{d/2}}{N_n}\cdot (\log n)^{3/2}\\
			&&\hspace*{1cm} + \frac{N_n \cdot(K^2+1)^d}{2 \cdot N_n^2} + \frac{c_1^3 \cdot (\log n)^3}{t_n}\\
			&&\leq F_n\left((\bw^*)^{(0)}\right)+ \nconst \cdot \frac{(K^2+1)^{d}}{{N_n}^{1/2}} \cdot (\log n)^{4L+4}.
		\end{eqnarray*}
		Thus we obtain
		\begin{eqnarray*}
			&&
			\EXP \left\{ T_{5,n} \right\}\\
			&&	= (1+ \epsilon)^2 \cdot \EXP \left\{\left( F_n\left( \bw^{(t_n)}\right) - \EXP\{|m(X)-Y|^2\}  \right)\cdot 1_{A_n}\right\}\\
			&&\qquad + ((1+\epsilon)^2 - 1) \cdot
			\EXP\left\{\EXP\{ |m(X)-Y|^2\} \cdot 1_{A_n} \right\} \\
			&&\leq (1+\epsilon)^2 \cdot \bigg(\EXP\Big\{F_n\left((\bw^*)^{(0)}\right)\cdot 1_{A_n}  + \const \Bigg(\frac{(K^2+1)^{d}}{{N_n}^{1/2}} \cdot (\log n)^{4L+4}\Bigg)\cdot 1_{A_n}\Big\}\\
			&&\hspace*{2cm} - \EXP\{ |m(X)-Y|^2 \} \cdot \PROB(A_n)\bigg)
			+ ((1+\epsilon)^2 - 1) \cdot
			\EXP\{ |m(X)-Y|^2 \}.
		\end{eqnarray*}
		Let $\tilde{A}_n$ be the event where the weight vector $\bw^{(0)}$
		satisfies
		\[
		| (\bw^{(0)})_{j_s,k,i}^{(l)}-\bw_{s,k,i}^{(l)}| \leq \log n
		\quad \mbox{for all } l \in \{0, \dots, L-1\},
		s \in \{1, \dots, N_n \cdot (K^2+1)^d\}
		\]
		for some weight vector $\bw$ which satisfies conditions
		(\ref{le6eq4})--(\ref{le6eq12})
		of Lemma \ref{le6} for $\bar{m}$. Then we get from the fifth step of the proof
		\begin{align*}
			\PROB(\tilde{A}_n)-\PROB(A_n) \leq \PROB\left\{ \frac{1}{n} \sum_{i=1}^{n}Y_i^2 > \beta_n^3\right\} \leq \frac{\nconst}{(\log n)^3}.
		\end{align*}
		This together with the fact that $(X_1,Y_1),\dots,(X_n,Y_n)$ are independent of $\tilde{A}_n$ yields
		\begin{eqnarray*}
			&&
			\EXP\{ F_n((\bw^*)^{(0)}) \cdot 1_{A_n} \}
			-
			\EXP\{ |m(X)-Y|^2 \} \cdot \PROB(A_n)
			\\
			&&
			\leq
			\EXP\Big\{\frac{1}{n} \sum_{i=1}^{n}|f_{((\bw^*)^{(0)})}(X_i)-Y_i|^2 \cdot 1_{\tilde{A}_n}
			\Big\}
			-
			\EXP\{ |m(X)-Y|^2 \} \cdot \PROB(\tilde{A}_n)\\&&\hspace*{2cm} + \EXP\{ |m(X)-Y|^2 \} \cdot (\PROB(\tilde{A}_n)-\PROB(A_n))
			\\
			&&
			\leq
			\EXP\left\{\EXP\Big\{\frac{1}{n} \sum_{i=1}^{n}|f_{((\bw^*)^{(0)})}(X_i)-Y_i|^2 \cdot 1_{\tilde{A}_n}\big| (\bw^*)^{(0)}
			\Big\}
			-
			\EXP\{ |m(X)-Y|^2 \} \cdot 1_{\tilde{A}_n} \right\}\\
			&&\hspace*{2cm} +  \frac{\nconst}{(\log n)^3}
			\\
			&&=	\EXP\left\{\left(\EXP\Big\{\frac{1}{n} \sum_{i=1}^{n}|f_{((\bw^*)^{(0)})}(X_i)-Y_i|^2 \big| (\bw^*)^{(0)}
			\Big\}
			-
			\EXP\{ |m(X)-Y|^2 \} \right) \cdot 1_{\tilde{A}_n} \right\}\\
			&&\hspace*{2cm} + \frac{\const}{(\log n)^3}
			\\
			&&\leq \EXP\left\{ \int |f_{((\bw^*)^{(0)})}(x)-m(x)|^2\PROB_X(dx)\cdot 1_{\tilde{A}_n}\right\} +  \frac{\const}{(\log n)^3}.
		\end{eqnarray*}

		Because of the choice of $\bar{m}$ and Lemma \ref{le6} we get for $K$ such large that ${supp(X) \subseteq [-K,K]^d}$
		\begin{align*}
			&\EXP \Big\{ \int |f_{((\bw^*)^{(0)})}(x)-m(x)|^2 \PROB_X(dx) \cdot 1_{\tilde{A}_n}
			\Big\} \\
			&\leq 2 \cdot \EXP \Big\{ \int |f_{((\bw^*)^{(0)})}(x)-\bar{m}(x)|^2
			\PROB_X(dx) \cdot 1_{\tilde{A}_n}\Big\}  + 2\int |\bar{m}(x)-m(x)|^2 \PROB_X(dx)\\
			&\leq 	\nconst \cdot \left( \frac{1}{K} + \frac{N_n^2\cdot K^{4d}}{n^2}
			+\left(\frac{ K^{2d}}{n} +1\right)^2 \cdot \PROB_X(\mathbb{R}^d\setminus [-K,K]^d)\right)+2\epsilon\\
			& \leq  \nconst \cdot \left( \frac{1}{K} + \frac{N_n^{2}\cdot K^{4d}}{n^2} 
			\right)+ 2 \epsilon.
		\end{align*}
		Due to the definition of $N_n$ we obtain
		\begin{align*}
			\frac{(K^2+1)^{d}}{{N_n}^{1/2}} \cdot (\log n)^{4L+4}\rightarrow 0 \qquad (n \rightarrow \infty).
		\end{align*}
		Summarizing the above results yields
		\begin{align*}
			\limsup_{n \rightarrow \infty}
			\EXP \left\{
			T_{5,n} 
			\right\}
			&\leq
			\nconst \cdot (1+\epsilon)^2 \cdot\left( \frac{1}{K} + 2\epsilon \right) + ((1+\epsilon)^2-1)\cdot\EXP\{|m(X)-Y|^2\}.
		\end{align*}
		
		In the {\it ninth step of the proof} we finish the proof
		of Theorem \ref{th1}. The results of steps 1,2,3,6,7 and 8 imply
		for $K\rightarrow \infty$
		\begin{align*}
			&\limsup_{n \rightarrow \infty}
			\EXP     \int | m_n(x)-m(x)|^2 \PROB_X (dx)\\
			& \leq \nconst \cdot  \left((1+\epsilon)^2 \cdot 2 \epsilon  + ((1+\epsilon)^2-1)\cdot\EXP\{|m(X)-Y|^2\} \right).
		\end{align*}
		With $\epsilon \rightarrow 0$ we get the assertion.
	\end{proof}

	\subsection{Auxiliary results for the proof of Theorem \ref{th2}}
	
	The following theorem is crucial for proving Theorem \ref{th2} and Theorem \ref{th3}.
	It is needed to analyze the rate of convergence of the over-parametrized deep neural network estimate.
	
	\begin{theorem}
		\label{th4}
		Let $n \in \N$ with $n\geq2$, let $(X,Y),(X_1,Y_1),\dots,(X_n,Y_n)$ be independent and identically distributed $\R^d\times\R$-valued random variables such that $supp(X)$ is bounded, that
		\begin{eqnarray}\label{th4eq-1}
			\EXP\left\{e^{c_3\cdot Y^2}\right\} < \infty
		\end{eqnarray}
		holds and that the corresponding regression function $m(x)=\EXP\{Y|X=x\}$ is bounded.
		Let $\sigma(x) = 1/(1+e^{-x})$ be the logistic squasher, let $K_n,L,r,t_n \in \N$, $M_n \geq 1$ and $\lambda_n,\tau>0$. Let $\tilde{K}_n \in \{1, \dots, K_n\}$,
		\[
		w_{k,i,j}^{(l)} \in
		[-20d \cdot (\log n)^2, 20d \cdot (\log n)^2]
		\quad \mbox{for }l=1, \dots, L, \ k=1, \dots \tilde{K}_n
		\]
		and
		\[
		w_{k,i,j}^{(0)}
		\in [-8d \cdot (\log n)^2 \cdot n^\tau,  8d \cdot (\log n)^2 \cdot n^\tau]
		\quad \mbox{for }k=1, \dots, \tilde{K}_n.
		\]Assume that
			\begin{eqnarray} \label{th4eq0}
			\sqrt{\sum_{k=1}^{\tilde{K}_n} |w_{1,1,k}^{(L)}|^2} \leq M_n,
		\end{eqnarray}
	and that
		\begin{eqnarray} \label{th4eq1}
			\bigg|\sum_{k=1}^{\tilde{K}_n} w_{1,1,k}^{(L)} \cdot f_{\bar{\bw},k,1}^{(L)}(x) \bigg| \leq \beta_n
		\end{eqnarray}
		holds for $x \in supp(X)$ and for all $\bar{\bw}$ which satisfy 
		\[
		|\bar{w}_{i,j,k}^{(l)}-w_{i,j,k}^{(l)}|\leq \log n 	\quad \mbox{for } l=0,\dots,L-1.
		\]
		Assume furthermore
		\begin{eqnarray}\label{th4eq2}
			\frac{K_n}{n^\kappa} \rightarrow 0 \qquad (n \rightarrow \infty)
		\end{eqnarray}
		for some $\kappa > 0$,
		\begin{eqnarray}\label{th4eq3}
			\frac{K_n}{\tilde{K}_n \cdot n^{r\cdot(d+1)\cdot \tau +1}} \rightarrow \infty \qquad (n \rightarrow \infty)
		\end{eqnarray}
		Define the estimate $m_{n}$ as in Section \ref{se2} with 
	\begin{eqnarray}	\label{th4eq4}
	\lambda_n = \frac{1}{t_n} \quad  \text{and} \quad  	t_n = \lceil \nconst \cdot L_n \rceil
\end{eqnarray}
for some $\const\geq 1$ and for some $L_n > 0$ which satisfies
\begin{align*}
	L_n \geq K_n^{3/2} \cdot (\log n)^{6L+2},
\end{align*}
		and assume
		\begin{eqnarray}\label{th4eq5}
			c_1\cdot c_3 \geq 2.
		\end{eqnarray}
		Then we have for any $\epsilon > 0 $
		\begin{eqnarray*}
			&&\EXP \int |m_{n}(x)-m(x)|^2\mathbf{P}_X(dx) 
			\leq \nconst \cdot \bigg( \frac{n^{\tau\cdot d + \epsilon}}{n} +  M_n^2 \cdot (\log n)^{4L+3/2} \\
			&& \hspace*{1cm} +\sup_{\substack{(\bar{w}_{i,j,k}^{(l)})_{i,j,k,l}\\|\bar{w}_{i,j,k}^{(l)}-w_{i,j,k}^{(l)}|\leq \log n \ (l=0,\dots,L-1)}} \int |\sum_{k=1}^{\tilde{K}_n} w_{1,1,k}^{(L)} \cdot f_{\bar{\bw},k,1}^{(L)}(x) - m(x)|^2 \PROB_X(dx)
			\bigg).
		\end{eqnarray*}
	\end{theorem}

	\begin{proof}
		\label{se4sub2}
		Let $A_n$ be the event that firstly the weight vector $\bw^{(0)}$
		satisfies
		\[
		| (\bw^{(0)})_{j_s,k,i}^{(l)}-\bw_{s,k,i}^{(l)}| \leq \log n
		\quad \mbox{for all } l \in \{0, \dots, L-1\},
		s \in \{1, \dots, \tilde{K}_n \}
		\]
		for some pairwise distinct $j_1,...,j_{\tilde{K}_n} \in \{1,...,K_n\}$ and that secondly
		\[
		\max_{i=1, \dots, n} |Y_i| \leq \sqrt{\beta_n}
		\]
		holds.
		Then we define the weight vectors
		$\bw^*$ for given $\tilde{\bw}$ by
		\[
		(\bw^*)_{k,i,j}^{(l)} = \tilde{\bw}_{k,i,j}^{(l)} \quad
		\mbox{for all } l=0,\dots, L-1,
		\]
		\[
		(\bw^*)_{1,1,j_k}^{(L)} = \tilde{\bw}_{1,1,k}^{(L)}  \quad \mbox{for all } k=1,\dots, \tilde{K}_n, 
		\]
		\[
		(\bw^*)_{1,1,k}^{(L)} = 0 \quad \mbox{for all } k \notin \{j_1,\dots,
		j_{\tilde{K}_n}\}
		\]
		and $(\bw^*)^{(0)}$ by
		\[
		((\bw^*)^{(0)})_{k,i,j}^{(l)} = (\bw_{k,i,j}^{(0)})^{(l)} \quad
		\mbox{for all } l=0,\dots, L-1,
		\]
		\[
		(\bw^*)^{(0)})_{1,1,j_k}^{(L)} = (\bw_{1,1,k}^{(0)})^{(L)} \quad \mbox{for all } k=1,\dots, \tilde{K}_n, 
		\]
		\[
		((\bw^*)^{(0)})_{1,1,k}^{(L)} = 0 \quad \mbox{for all } k \notin \{j_1,\dots,
		j_{\tilde{K}_n}\}.
		\]
		
		In the following we set
		\[
		m_{\beta_n}(x)=\EXP\{ T_{\beta_n} Y | X=x \}.
		\]
		Furthermore, we assume w.l.o.g. that $n$ is sufficiently large and that $\|m\|_\infty \leq \beta_n$ holds. 
		
		In the {\it first step of the proof} we decompose the  $L_2$ error of $m_n$ in a sum of several terms. We have
		\begin{eqnarray*}
			&&
			\int | m_{n}(x)-m(x)|^2 \PROB_X (dx)
			\\
			&&
			=
			\left(
			\EXP \left\{ |m_{n}(X)-Y|^2 | \D_n \right\}
			-
			\EXP \{ |m(X)-Y|^2\}
			\right)
			\cdot 1_{A_n}\\
			&& \hspace*{2cm}+
			\int | m_{n}(x)-m(x)|^2 \PROB_X (dx)
			\cdot 1_{A_n^c}
			\\
			&&
			=
			\Big[
			\EXP \left\{ |m_{n}(X)-Y|^2 | \D_n \right\}
			-
			\EXP \{ |m(X)-Y|^2\}
			\\
			&&
			\hspace*{2cm}
			- \left(
			\EXP \left\{ |m_{n}(X)-T_{\beta_n} Y|^2 | \D_n \right\}
			-
			\EXP \{ |m_{\beta_n}(X)- T_{\beta_n} Y|^2\}
			\right)
			\Big] \cdot 1_{A_n}
			\\
			&&
			\quad +
			\Big[
			\EXP \left\{ |m_{n}(X)-T_{\beta_n} Y|^2| \D_n \right\}
			-
			\EXP \{ |m_{\beta_n}(X)- T_{\beta_n} Y|^2\}
			\\
			&&
			\hspace*{2cm}
			-
			2 \cdot \frac{1}{n} \sum_{i=1}^n
			\left(
			|m_{n}(X_i)-T_{\beta_n} Y_i|^2
			-
			|m_{\beta_n}(X_i)- T_{\beta_n} Y_i|^2
			\right)
			\Big] \cdot 1_{A_n}
			\\
			&&
			\quad
			+\Big[
			2 \cdot \frac{1}{n} \sum_{i=1}^n
			|m_{n}(X_i)-T_{\beta_n} Y_i|^2
			-
			2 \cdot \frac{1}{n} \sum_{i=1}^n
			|m_{\beta_n}(X_i)- T_{\beta_n} Y_i|^2
			\\
			&&
			\hspace*{2cm}
			- \left(
			2 \cdot \frac{1}{n} \sum_{i=1}^n
			|m_{n}(X_i)-Y_i|^2
			-
			2 \cdot \frac{1}{n} \sum_{i=1}^n
			|m(X_i)- Y_i|^2
			\right)
			\Big] \cdot 1_{A_n}
			\\
			&&
			\quad
			+
			\Big[
			2 \cdot \frac{1}{n} \sum_{i=1}^n
			|m_{n}(X_i)-Y_i|^2
			-
			2 \cdot \frac{1}{n} \sum_{i=1}^n
			|m(X_i)- Y_i|^2
			\Big] \cdot 1_{A_n}
			\\
			&&
			\quad
			+
			\int | m_{n}(x)-m(x)|^2 \PROB_X (dx)
			\cdot 1_{A_n^c}
			\\
			&&
			=: \sum_{j=1}^5 T_{j,n}.
		\end{eqnarray*}

		In the {\it second step of the proof} we show
		\[
		\EXP T_{1,n} \leq \nconst \cdot \frac{\log n}{n} \quad \mbox{and} \quad \EXP T_{3,n} \leq \nconst \cdot \frac{\log n}{n}.
		\]
		This follows as in the proof of Lemma 1 in \citet{BaKo19}.
		
		In the {\it third step of the proof} we show
		\[
		\EXP T_{5,n} \leq \nconst\cdot \frac{(\log n)^2}{n}.
		\]
		Due to the definition of $m_{n}$ and the assumption that $\|m\|_\infty \leq \beta_n$, it holds
		\[\int |m_{n}(x)-m(x)|^2 \PROB_X (dx) \leq
		4 \cdot {c_1}^2 \cdot (\log
		n)^2.\]
		Thus it suffices to show
		\begin{equation}
			\label{pth1eq1}
			\PROB(A_n^c) \leq \frac{\nconst}{n}.
		\end{equation}
		To do this, we first bound the probability that the weights in the first of the $K_n$ fully connected neural networks differ by at most $\log n$ from ($\bw^{(l)}_{1,i,j})_{i,j,l:l<L}$ in all components. For large $n$, this probability is bounded from below by
		\begin{eqnarray*}
			\left( \frac{\log n}{40d\cdot (\log n)^2}
			\right)^{r \cdot (r+1) \cdot (L-2)}
			\cdot
			\left(
			\frac{\log n}{16d \cdot (\log n)^2 \cdot n^\tau}
			\right)^{r \cdot (d+1)}
			&\geq&
			n^{- r \cdot (d+1) \cdot \tau-0.5}.
		\end{eqnarray*}
		Then the probability that none of the first $n^{r \cdot (d+1) \cdot
			\tau +1}$ neural networks satisfies this condition is
		bounded above by
		\begin{eqnarray*}
			(1 -  n^{- r \cdot (d+1) \cdot \tau-0.5}) ^{n^{r \cdot (d+1) \cdot
					\tau +1}}
			&\leq& \left(\exp \left(-  n^{- r \cdot (d+1) \cdot \tau-0.5}\right)\right) ^{n^{r \cdot (d+1) \cdot\tau +1}}\\
			&=&\exp( -  n^{0.5}).
		\end{eqnarray*}
		Assumption (\ref{th4eq3}) implies $K_n \geq n^{r \cdot (d+1) \cdot
			\tau +1} \cdot \tilde{K}_n$ for large $n$. Thus we can apply this construction successively for all $\tilde{K}_n$ weights $((\bw^{(0)})^{(l)}_{k,i,j})_{i,j,l:l<L}$. The probability that there exists $k \in \{1, \dots, \tilde{K}_n\}$ such that none of the $K_n$ weight vectors of the fully connected neural network differs from $(\bw_{k,i,j}^{(l)})_{i,j,l:l<L}$ by at most $\log n$ is then for large $n$ bounded from above by
		\begin{eqnarray*}
			&&
			\tilde{K}_n \cdot \exp( -  n^{0.5})
			\leq  n^\kappa \cdot \exp( - n^{0.5})
			\leq \frac{\nconst}{n}.
		\end{eqnarray*}
		Hence, for large $n$ it is
		\begin{eqnarray*}
			\PROB(A_n^c)
			&\leq&
			\frac{\const}{n}
			+
			\PROB\{ \max_{i=1, \dots, n} |Y_i| > \sqrt{\beta_n}
			\}
			\leq
			\frac{\const}{n} + n \cdot\PROB\{  |Y| > \sqrt{\beta_n}
			\}
			\\
			& \leq &
			\frac{\const}{n} + n \cdot
			\frac{\EXP\{ \exp(c_3 \cdot Y^2)\}}{ \exp( c_3 \cdot \beta_n)}
			\leq
			\frac{\nconst}{n},
		\end{eqnarray*}
		where the last inequality holds because of (\ref{th4eq-1}) and (\ref{th4eq5}).
		
		In the {\it fourth step of the proof} we show that the assumptions (\ref{le1eq1}) - (\ref{le1eq3}) of Lemma \ref{le1} are satisfied which means that
		\begin{equation*}
			\| (\nabla_{\bw} F)(\bw) \| \leq L_n
		\end{equation*}
		for all $\bw \in S:=\left\{
		\bv  \, : \,
		\|\bv - \bw^{(0)}\| \leq
		2 \cdot \sqrt{F(\bw^{(0)})} + 1
		\right\}$,
		\begin{equation*}
			\| (\nabla_{\bw} F)(\bw) - (\nabla_{\bw} F)(\bar{\bw}) \| \leq
			L_n \cdot \|\bw-\bar{\bw}\|
		\end{equation*}
		for all $\bw, \bar{\bw} \in S$ 
		and
		\begin{align*}
			&|F(\bw^*)-F((\bw^*)^{(0)})|\\
			& \leq D_n \cdot \|((\bw^*)^{(L)}_{1,1,k})_{k=1,\dots,K_n}\| \cdot \|(\bw_{i,j,k}^{(l)})_{i,j,k,l:l<L}-((\bw^{(0)})_{i,j,k}^{(l)})_{i,j,k,l:l<L}\|
		\end{align*}
		for all 
		\begin{align*}
			&(\bw_{i,j,k}^{(l)})_{i,j,k,l:l<L} \in \tilde{S}\\
			&\hspace*{1cm}:= \left\{(\bar{\bw}_{i,j,k}^{(l)})_{i,j,k,l:l<L} \, : \, \|(\bar{\bw}_{i,j,k}^{(l)})_{i,j,k,l:l<L}-((\bw^{(0)})_{i,j,k}^{(l)})_{i,j,k,l:l<L}\| \leq \sqrt{2 \cdot F(\bw^{(0)})} \right\}
		\end{align*} hold, 
		if $A_n$ holds.
		
		If $A_n$ holds, then we have
		\[
		F_n(\bw^{(0)})
		=
		\frac{1}{n} \sum_{i=1}^n Y_i^2 \leq \beta_n.
		\]
		
		Let $\bw \in S$. Then we have
		\begin{align*}
			\|(\bw_{i,j,k}^{(l)})_{i,j,k,l:1\leq l<L}\|_\infty &\leq \|\bw-\bw^{(0)}\| +  \|((\bw^{(0)})_{i,j,k}^{(l)})_{i,j,k,l:1\leq l<L}\|_\infty\\
			&\leq2 \cdot  \sqrt{ F_n(\bw^{(0)})} + 1+  \nconst \cdot (\log n)^2\\
			&\leq \nconst \cdot (\log n)^2
		\end{align*}
		and 
		\begin{align*}
			\|(\bw^{(L)}_{1,1,k})_{k=1,\dots,K_n}\|_\infty &\leq \|\bw-\bw^{(0)}\| +  \|((\bw^{(0)})^{(L)}_{1,1,k})_{k=1,\dots,K_n}\|_\infty\\
			&\leq 2 \cdot  \sqrt{ F_n(\bw^{(0)})} + 1\\
			&\leq \nconst \cdot (\log n)^{1/2}.
		\end{align*}
		Hence (\ref{le2eq1})-(\ref{le3eq3}) are satisfied for $B_n=\mmconst{1} \cdot (\log n)^2$ and $\gamma_n^*=\const \cdot (\log n)^{1/2}$. By Lemma \ref{le2} and Lemma \ref{le3} we get that  (\ref{le1eq1}) and (\ref{le1eq2}) are satisfied provided that $L_n\geq K_n^{3/2} \cdot (\log n)^{6L+2}$.
		
		Furthermore, let $\tilde{\bw}$ such that
		$(\tilde{\bw}_{i,j,k}^{(l)})_{i,j,k,l:l<L} \in \tilde{S}.$
		Then we obtain as in the proof of Theorem \ref{th1} 
		\begin{align*}
			&|F_n(\bw^*)-F_n((\bw^*)^{(0)})|\\
			& \leq \Bigg(\frac{2}{n} \sum_{i=1}^{n}  \left(f_{\bw^*}(X_i)+f_{(\bw^*)^{(0)}}(X_i)\right)^2 + \frac{8}{n} \sum_{i=1}^{n} Y_i^2\Bigg)^{1/2}\\
			& \hspace*{2cm} \Bigg(\frac{1}{n} \sum_{i=1}^{n}	\left(\sum_{k=1}^{K_n}\left|(\bw^*)_{1,1,k}\right|^2\cdot \sum_{k=1}^{K_n}\left|f_{\bw^*,{k,1}}^{(L)}(X_i) - f_{(\bw^*)^{(0)},{k,1}}^{(L)}(X_i)\right|^2\right)\Bigg)^{1/2}.
		\end{align*}
		With
		\begin{align*}
			&\Bigg(\frac{1}{n} \sum_{i=1}^{n}	\left(\sum_{k=1}^{K_n}\left|(\bw^*)_{1,1,k}\right|^2\cdot \sum_{k=1}^{K_n}\left|f_{\bw^*,k,1}^{(L)}(X_i) - f_{(\bw^*)^{(0)},{k,1}}^{(L)}(X_i)\right|^2\right)\Bigg)^{1/2}\\
			&\leq \|((\bw^*)^{(L)}_{1,1,k})_{k=1,\dots,K_n}\| \cdot \nconst \cdot (\log n)^{2L} \cdot \| (\tilde{\bw}^{(l)}_{i,j,k})_{i,j,k,l:l<L} - ((\tilde{\bw}^{(0)})_{i,j,k}^{(l)})_{i,j,k,l:l<L}\|
		\end{align*}
		and since $(\tilde{\bw}_{i,j,k}^{(l)})_{i,j,k,l:l<L} \in \tilde{S}$ we get
		\begin{align*}
			&\Bigg(\frac{2}{n} \sum_{i=1}^{n}  \left(f_{\bw^*}(X_i)+f_{(\bw^*)^{(0)}}(X_i)\right)^2 + \frac{8}{n} \sum_{i=1}^{n} Y_i^2\Bigg)^{1/2}\\
			&\leq \Bigg( 4  \cdot \sum_{k=1}^{K_n}\left|(\bw^*)_{1,1,k}^{(L)}\right|^2\cdot \max_{i=1,\dots,n} \sum_{k=1}^{K_n}\left|f^{(L)}_{\bw^*,k,1}(X_i) - f^{(L)}_{(\bw^*)^{(0)},k,1}(X_i)\right|^2\\
			&\hspace*{2cm}+ 16 \cdot \frac{1}{n} \sum_{i=1}^{n}  f_{(\bw^*)^{(0)}}(X_i)^2 + \frac{8}{n} \sum_{i=1}^{n} Y_i^2\Bigg)^{1/2}\\
			&\leq \Bigg(4 \cdot \sum_{k=1}^{K_n}\left|(\bw^*)_{1,1,k}^{(L)}\right|^2\cdot \nconst \cdot (\log n)^{4L} \cdot  \| ((\bw^*)^{(l)}_{i,j,k})_{i,j,k,l:l<L} - (((\bw^*)^{(0)})_{i,j,k}^{(l)})_{i,j,k,l:l<L}\|^2\\
			&\hspace*{1cm} +16\cdot\sum_{i=1}^{n} f_{(\bw^*)^{(0)}}(X_i)^2 + \frac{8}{n} \sum_{i=1}^{n} Y_i^2\Bigg)^{1/2}\\
			&\leq \Bigg(\nconst \cdot M_n^2 \cdot (\log n)^{4L+1}+16 \cdot \beta_n^2  + 8\cdot\beta_n\Bigg)^{1/2}\\
			& \leq \nconst \cdot M_n\cdot (\log n)^{2L+1} .
		\end{align*}
		
		Summarizing these steps yields
		\begin{align*}
			&|F_n(\bw^*)-F_n((\bw^*)^{(0)})|\\ 
			&\leq \nconst \cdot (\log n)^{4L+1} \cdot M_n \cdot \|((\bw^*)^{(L)}_{1,1,k})_{k=1,\dots,K_n}\| \cdot \| (\tilde{\bw}^{(l)}_{i,j,k})_{i,j,k,l:l<L} - ((\tilde{\bw}^{(0)})_{i,j,k}^{(l)})_{i,j,k,l:l<L}\|
		\end{align*}
		Thus (\ref{le1eq3}) is satisfied with 
		\[D_n =\const \cdot M_n \cdot  (\log n)^{4L+1}.
		\]
		
		Let $\epsilon >0$ be arbitrary.
		In the {\it fifth step of the proof} we show
		\[
		\EXP T_{2,n} \leq
		\nconst \cdot
		\frac{ n^{\tau \cdot d + \epsilon}}{n}
		.
		\]
		Let $\W_n$ be the set of all weight vectors
		$(w_{i,j,k}^{(l)})_{i,j,k,l}$ which satisfy
		\[
		| w_{1,1,k}^{(L)}| \leq \nconst\cdot (\log n)^2 \quad (k=1, \dots, K_n),
		\]
		\[
		|w_{i,j,k}^{(l)}| \leq (20d+1) \cdot (\log n)^2 \quad (l=1, \dots, L-1)
		\]
		and
		\[
		|w_{i,j,k}^{(0)}| \leq (8d+1) \cdot (\log n)^2\cdot n^\tau.
		\]
		
		From Lemma \ref{le1} we know for $n$ large that 
		\begin{align*}
			\|((\bw^{(t_n)})^{(L)}_{1,1,k})_{k=1,\dots,K_n}-((\bw^{(0)})^{(L)}_{1,1,k})_{k=1,\dots,K_n}\| \leq \sqrt{2 F_n(\bw^{(0)})} \leq (\log n)^2
		\end{align*}
		and 
		\begin{align*}
			\|((\bw^{(t_n)})^{(l)}_{i,j,k})_{i,j,k,l:l<L} - ((\bw^{(0)})_{i,j,k}^{(l)})_{i,j,k,l:l<L}\| \leq \sqrt{2 F_n(\bw^{(0)})} \leq (\log n)^2.
		\end{align*}
		This and the initial choice of $\bw^{(0)}$ imply that we have on $A_n$ for $n$ large
		\[
		\bw^{(t)} \in \W_n \quad (t=0, \dots, t_n).
		\]
		
		Thus, for any $u>0$ we get
		\begin{eqnarray*}
			&&
			\PROB \{ T_{2,n} > u \}\\
			&&\leq\PROB \Bigg\{ \exists f \in \F_n : \EXP \left(\left|\frac{f(X)}{\beta_n} - \frac{T_{\beta_n}Y}{\beta_n}\right|^2\right)-\EXP \left(\left|\frac{m_{\beta_n}(X)}{\beta_n} - \frac{T_{\beta_n}Y}{\beta_n}\right|^2\right)\\
			&&\hspace*{3cm}-\frac{1}{n} \sum_{i=1}^n\left(\left|\frac{f(X_i)}{\beta_n} - \frac{T_{\beta_n}Y_i}{\beta_n}\right|^2-\left|\frac{m_{\beta_n}(X_i)}{\beta_n} - \frac{T_{\beta_n}Y_i}{\beta_n}\right|^2\right)\Bigg\}\\
			&&\hspace*{2cm}> \frac{1}{2} \cdot\left(\frac{u}{\beta_n^2}+\EXP \left(\left|\frac{f(X)}{\beta_n} - \frac{T_{\beta_n}Y}{\beta_n}\right|^2\right)-\EXP \left(\left|\frac{m_{\beta_n}(X)}{\beta_n} - \frac{T_{\beta_n}Y}{\beta_n}\right|^2\right)\right),
		\end{eqnarray*}
		where
		\[
		\F_n = \left\{ T_{\beta_n} f_\bw:\bw \in \W_n \right\}.
		\]
		Application of Lemma \ref{le4} yields for $x_1^n\in supp(X)$ and $\alpha=\nconst$
		\begin{eqnarray*}
			&&\Nu_1 \left(\delta , \left\{\frac{1}{\beta_n} \cdot f : f \in \F_n\right\}, x_1^n\right)\leq\Nu_1 \left(\delta \cdot \beta_n , \F_n, x_1^n\right)\\
			&&\leq\left(\frac{\nconst \cdot \beta_n}{\delta\cdot \beta_n}\right)^{\nconst \cdot (\log n)^{2d} \cdot n^{\tau \cdot d} \cdot (\log n)^{2 \cdot (L-1) \cdot d}\cdot\left(\frac{K_n \cdot (\log n)^2}{\delta \cdot \beta_n}\right)^{d/k} + \nconst}.
		\end{eqnarray*}
		For $\delta>1/n$ and $k$ large enough, we obtain
		\[
		\Nu_1 \left(
		\delta , \left\{
		\frac{1}{\beta_n} \cdot f : f \in \F_n
		\right\}
		, x_1^n
		\right)
		\leq
		\nconst \cdot n^{ \nconst \cdot n^{\tau \cdot d + \epsilon/2}}.
		\]
		This together with Theorem 11.4 in \citet{GyKoKrWa02} leads for $u
		\geq 1/n$ to
		\[
		\PROB\{T_{2,n}>u\}
		\leq
		14 \cdot
		\mconst \cdot n^{ \const \cdot n^{\tau \cdot d +  \epsilon/2}}
		\cdot
		\exp \left(
		- \frac{n}{5136 \cdot \beta_n^2} \cdot u
		\right).
		\]
		For $\epsilon_n \geq 1/n$ we can conclude
		\begin{eqnarray*}
			\EXP \{ T_{2,n} \}
			& \leq &
			\epsilon_n + \int_{\epsilon_n}^\infty \PROB\{ T_{2,n}>u \} \, du
			\\
			& \leq &
			\epsilon_n
			+
			14 \cdot
			\mconst \cdot n^{\const \cdot n^{\tau \cdot d +
					\epsilon/2}}
			\cdot
			\exp \left(
			- \frac{n}{5136 \cdot \beta_n^2} \cdot \epsilon_n
			\right)
			\cdot
			\frac{5136 \cdot \beta_n^2}{n}.
		\end{eqnarray*}
		Setting
		\[
		\epsilon_n = \frac{5136 \cdot \beta_n^2}{n}
		\cdot
		\const
		\cdot
		n^{\tau \cdot d +
			\epsilon/2}
		\cdot \log n
		\]
		yields the assertion of the fifth step of the proof.

		In the {\it sixth step of the proof} we show
		\begin{eqnarray*}
			&&
			\EXP \{ T_{4,n} \} \\
			&&\leq
			\nconst \cdot \Bigg(
			\sup_{
				(\bar{w}_{i,j,k}^{(l)})_{i,j,k,l} :
				\atop
				|\bar{w}_{i,j,k}^{(l)}-w_{i,j,k}^{(l)}| \leq \log n
				\, (l=0, \dots, L-1)
			}
			\int
			|
			\sum_{k=1}^{\tilde{K}_n}
			w_{1,1,k}^{(L)} \cdot f_{\bar{\bw},k,1}^{(L)}(x)-m(x)|^2 \PROB_X (dx)
			\\
			&&
			\hspace*{2cm}+
			\nconst  \cdot \frac{\log n}{n}
			+
			\nconst  \cdot
			\frac{ n^{\tau \cdot d + \epsilon}}{n} + \nconst \cdot  M_n^2 \cdot (\log n)^{4L+3/2}.
		\end{eqnarray*}
		Since
		\[
		|T_{\beta_n} z - y| \leq |z-y|
		\quad \mbox{for } |y| \leq \beta_n
		\]
		we obtain
		\begin{eqnarray*}
			&&
			T_{4,n}/2\\
			&&=\Big[ \frac{1}{n} \sum_{i=1}^n|m_{n}(X_i)-Y_i|^2-\frac{1}{n} \sum_{i=1}^n|m(X_i)- Y_i|^2\Big] \cdot 1_{A_n}\\
			&&\leq\Big[\frac{1}{n} \sum_{i=1}^n|f_{\bw^{(t_n)}}(X_i)-Y_i|^2-\frac{1}{n} \sum_{i=1}^n |m(X_i)- Y_i|^2 \Big] \cdot 1_{A_n}\\
			&&=\big[ F_n(\bw^{(t_n)})-\frac{1}{n} \sum_{i=1}^n|m(X_i)- Y_i|^2 \Big] \cdot 1_{A_n}.
		\end{eqnarray*}
		
		The application of Lemma \ref{le1} implies
		\begin{eqnarray*}
			&&\EXP\{ T_{4,n}/2 \}
			\\
			&&
			\leq \EXP\Bigg\{\Bigg[F_n\left((\bw^*)^{(0)}\right)+ D_n\cdot\|(((\bw^*)_{1,1,k}^{(L)})_{k=1,\dots,K_n}\| \cdot \sqrt{2 \cdot F_n\left(\bw^{(0)}\right)}\\
			&&\hspace*{0.4cm}
			+\frac{\|((\bw^*)_{1,1,k}^{(L)})_{k=1,\dots,K_n}-(((\bw^*)^{(0)})_{1,1,k}^{(L)})_{k=1,\dots,K_n}\|^2}{2}+ \frac{F_n\left(\bw^{(0)}\right)}{t_n}\ \\
			&& \hspace*{0.4cm} -\frac{1}{n} \sum_{i=1}^n
			|m(X_i)- Y_i|^2
			\Bigg] \cdot 1_{A_n}\Bigg\} \\
			&&
			\leq 2 \cdot
			\Bigg(
			\sup_{
				(\bar{w}_{i,j,k}^{(l)})_{i,j,k,l} :
				\atop
				|\bar{w}_{i,j,k}^{(l)}-w_{i,j,k}^{(l)}| \leq \log n
				\, (l=0, \dots, L-1)
			}
			\int
			|
			\sum_{k=1}^{\tilde{K}_n}
			w_{1,1,k}^{(L)}  \cdot f_{\bar{\bw},k,1}^{(L)}(x)-m(x)|^2 \PROB_X (dx)
			\Bigg)
			\\
			&&
			\hspace*{0.4cm}
			+\EXP \Bigg\{
			\Bigg(
			F_n\left((\bw^*)^{(0)}\right)+ D_n\cdot\|(((\bw^*)_{1,1,k}^{(L)})_{k=1,\dots,K_n}\| \cdot \sqrt{2 \cdot F_n\left(\bw^{(0)}\right)}\\
			&&\hspace*{0.4cm}
			+\frac{\|((\bw^*)_{1,1,k}^{(L)})_{k=1,\dots,K_n}-(((\bw^*)^{(0)})_{1,1,k}^{(L)})_{k=1,\dots,K_n}\|^2}{2}+ \frac{F_n\left(\bw^{(0)}\right)}{t_n} \\
			&&\hspace*{0.4cm}-\frac{1}{n} \sum_{i=1}^n
			|m(X_i)- Y_i|^2 
			\\
			&&
			\hspace*{0.4cm}
			- 2 \cdot  \left(
			\EXP\{ |\sum_{k=1}^{{K}_n}
			(\bw^*)_{1,1,k}^{(L)} \cdot f_{(\bw^*),j,1}^{(L)}(X)-Y|^2 | \D_n,\bw^{(0)}\}
			-\EXP\{|m(X)-Y|^2\} \right)\Bigg)
			1_{A_n} \Bigg\}.
		\end{eqnarray*}
		Due to step 4 we have 
		\begin{align*}
			D_n = \mmconst{12} \cdot M_n \cdot (\log n)^{4L+ 1} .
		\end{align*}
		
		Using the same arguments as in step 2 and 5 of the proof we obtain
		\begin{eqnarray*}
			&&
			\EXP \Bigg\{
			\Bigg(
			F_n((\bw^*)^{(0)})-\frac{1}{n} \sum_{i=1}^n
			|m(X_i)- Y_i|^2 \\
			&& \hspace*{0.4cm} +
			\nconst \cdot M_n \cdot (\log n)^{4L+1} \cdot \|(((\bw^*)_{1,1,k}^{(L)})_{k=1,\dots,K_n}\| \cdot \sqrt{2 \cdot F_n\left(\bw^{(0)}\right)}\\
			&&\hspace*{0.4cm}
			+\frac{\|((\bw^*)_{1,1,k}^{(L)})_{k=1,\dots,K_n}-(((\bw^*)^{(0)})_{1,1,k}^{(L)})_{k=1,\dots,K_n}\|^2}{2}+ \frac{F_n\left(\bw^{(0)}\right)}{t_n} 
			\\
			&&
			\hspace*{0.4cm}
			- 2 \cdot  \left(
			\EXP\{ |\sum_{k=1}^{{K}_n}
			(\bw^*)_{1,1,k}^{(L)} \cdot f_{(\bw^*),j,1}^{(L)}(X)-Y|^2 | \D_n,\bw^{(0)}\}
			-\EXP\{|m(X)-Y|^2\} \right) \Bigg) 
			1_{A_n} \Bigg\}\\
			&&	\leq
			\nconst \cdot \frac{\log n}{n}
			+ \const \cdot  M_n^2 \cdot (\log n)^{4L+3/2} + \frac{M_n^2}{2} + \frac{c_1 \cdot \log n}{t_n} +
			\nconst \cdot
			\frac{ n^{\tau \cdot d + \epsilon}}{n} .
		\end{eqnarray*}
		
		This implies
		\begin{eqnarray*}
			&&
			\EXP\{ T_{4,n}/2 \} \\
			&&
			\leq
			2 \cdot
			\Bigg(
			\sup_{
				(\bar{w}_{i,j,k}^{(l)})_{i,j,k,l} :
				\atop
				|\bar{w}_{i,j,k}^{(l)}-w_{i,j,k}^{(l)}| \leq \log n
				\, (l=0, \dots, L-1)
			}
			\int
			|
			\sum_{k=1}^{\tilde{K}_n}
			w_{1,1,k}^{(L)} \cdot f_{\bar{\bw},k,1}^{(L)}(x)-m(x)|^2
			\PROB_X (dx)
			\Bigg)
			\\
			&&
			\hspace*{2cm}
			+
			\mmconst{3} \cdot \frac{\log n}{n}
			+
			\mmconst{2} \cdot
			\frac{n^{\tau \cdot d + \epsilon}}{n} +\nconst \cdot  M_n^2 \cdot (\log n)^{4L+3/2} .
		\end{eqnarray*}
		Summarizing the above results we get the assertion.

			\end{proof}

To prove Theorem 2, we use the following lemma, which provides another bound on the approximation error. Furthermore, it ensures that the outer weights are sufficiently small.

\begin{lemma}
	\label{le7}
	Let $1/2 \leq p \leq 1$, $C>0$,
	let $f:\Rd \rightarrow \R$ be a $(p,C)$--smooth function and let
	$X$ be a $\Rd$-valued random variable with $supp(X) \subseteq
	[0,1]^d$.
	Let $l \in \N$, $0<\delta<1/2$ with
	\begin{equation}
		\label{le7eq1}
		\nconst \cdot \delta \leq \frac{1}{2^l} \leq \nconst \cdot \delta
	\end{equation}
	and let $L,r,s \in \N$ with $L \geq 2$ and $r \geq 2d$. Furthermore, let
	\[
	\tilde{K}_n \geq \left( l \cdot (2^l+1)^{2d}+1 \right)^3.
	\]
	Then there exist
	\[
	w_{k,i,j}^{(l)} \in
	[-20d \cdot (\log n)^2, 20d \cdot (\log n)^2]
	\quad \text{for } l=1, \dots, L, k=1, \dots \tilde{K}_n
	\]
	and
	\[
	w_{k,i,j}^{(0)}
	\in \left[-\frac{8 \cdot d \cdot (\log n)^2}{\delta}, \frac{8 \cdot d \cdot
		(\log n)^2}{\delta} \right]
	\quad  \text{for } k=1, \dots, \tilde{K}_n
	\]
	such that
	for all $\bar{\bw}$ satisfying
	$|\bar{w}_{i,j,k}^{(l)}-w_{i,j,k}^{(l)}| \leq \log n$
	$(l=0, \dots, L-1)$ we have for $n$ sufficiently large
	\begin{eqnarray}
		\label{le8eq3}
		&&
		\int
		|
		\sum_{k=1}^{\tilde{K}_n}
		w_{1,1,k}^{(L)} \cdot f_{\bar{\bw},k,1}^{(L)}(x)-f(x)|^2 \PROB_X (dx)
		\nonumber \\
		&&
		\hspace*{2cm}
		\leq
		\nconst \cdot \left(
		l^2 \cdot \delta + \delta^{2p}
		+
		\frac{l \cdot (2^l+1)^{2d}}{n^s}
		\right)
		,
	\end{eqnarray}
	\begin{equation}
		\label{le8eq1}
		|
		\sum_{k=1}^{\tilde{K}_n}
		w_{1,1,k}^{(L)} \cdot f_{\bar{\bw},k,1}^{(L)}(x)| \leq
		\nconst
		\cdot
		\left( 1 +
		\frac{(2^l+1)^{2d}}{n^s}
		\right) \qquad (x \in [0,1]^d)
	\end{equation}
	and
	\begin{equation}
		\label{le8eq4}
		\sum_{k=1}^{\tilde{K}_n} |w_{1,1,k}^{(L)}|^2 \leq \frac{\nconst}{2^{2    \cdot d
				\cdot l}} .
	\end{equation}
\end{lemma}

\begin{proof}
	
	The proof follows from the proof of Lemma 7 in \citet{KoKr22a}.
\end{proof}

\subsection{Proof of Theorem \ref{th2}}

Let $l=\lfloor \frac{1}{1+d} \log_2 n \rfloor$. Then condition (\ref{le7eq1}) holds for $\delta= n^{-1/(1+d)}$. Set 
\begin{align*}
	\tilde{K}_n= (l \cdot(2^l+1)^{2d}+1)^3 \approx \nconst \cdot (\log n)^3 \cdot n^{\frac{6d}{1+d}}, \qquad N_n=n^{\nconst}
		\end{align*}
 and define the weight vector $\bw$ as in Lemma \ref{le7}.
Then we obtain by Lemma \ref{le7} that assumption (\ref{th4eq0}) is satisfied for $M_n=\frac{\nconst}{n^{d/(d+1)}}$. Assumption (\ref{th4eq1}) follows directly from (\ref{le8eq1}) of  Lemma \ref{le7} for $s$ sufficiently large.

By Theorem \ref{th4}, $\tau=\frac{1}{1+d}$ and Lemma \ref{le7} where $s$ is sufficiently large we get for large $n$
\begin{eqnarray*}
	&&
	\EXP \int | m_{n}(x)-m(x)|^2 \PROB_X (dx)
	\leq
	\nconst \cdot
	\Bigg(
	\frac{ n^{\frac{1}{1+d} \cdot d + \epsilon}}{n} + \nconst \cdot  \frac{(\log n)^{4L+3/2}}{n^{\frac{2d}{d+1}}}
	\\
	&&
	\hspace*{2cm}
	+
	\sup_{
		(\bar{w}_{i,j,k}^{(l)})_{i,j,k,l} :
		\atop
		|\bar{w}_{i,j,k}^{(l)}-w_{i,j,k}^{(l)}| \leq \log n
		\, (l=0, \dots, L-1)
	}
	\int
	|
	\sum_{k=1}^{\tilde{K}_n}
	w_{1,1,k}^{(L)} \cdot f_{\bar{\bw},k,1}^{(L)}(x)-m(x)|^2 \PROB_X (dx)	\Bigg)\\	
	&&
	\leq
	\nconst \cdot
	\left(
	\frac{ n^{\frac{1}{1+d} \cdot d + \epsilon}}{n} + \frac{(\log n)^{4L+3/2}}{ n^{\frac{2d}{1+d}}} + \frac{(\log
	n)^2} {n^{\frac{1}{1+d}}} + \frac{1}{n^{\frac{2p}{1+d}}} + \frac{\log n\cdot n^{\frac{2d}{1+d}}}{n^s} 
	\right)
	\\
	&&
	\leq
	\nconst \cdot
	n^{
		- \frac{1}{1+d} + \epsilon
	}
\end{eqnarray*}
for $s$ sufficiently large.
\hfill $\Box$

\subsection{Auxiliary results for the proof of Theorem \ref{th3}}

In order to prove Theorem 3, we use the following lemma, which controls the complexity of a set of over-parametrized deep neural networks for interaction models.

\begin{lemma}
	\label{le8}
	Let $\alpha \geq 1$, $\beta>0$ and let $A,B,C \geq 1$.
	Let $\sigma:\R \rightarrow \R$ be $k$-times differentiable
	such that all derivatives up to order $k$ are bounded on $\R$.
	Let $\F$
	be the set of all functions
	\[
	f_{\bw}(x) = \sum_{I \subseteq \{1, \dots, d\} \, : \, |I|=d^* } f_{\bw_I} (x_I)
	\]
	where $f_{\bw_I}$ is defined by (\ref{se2eq1})--(\ref{se2eq3}) with
	$d$ replaced by $d^*$ and weight vector $\bw_I$,
	\[
	\bw= \left( \bw_I \right)_{I \subseteq \{1, \dots, d\} \, : \, |I|=d^* },
	\]
	and where for any
	$I \subseteq \{1, \dots, d\}$ with $|I|=d^*$
	the weight vector $\bw_I$ satisfies
	\begin{equation}
		\label{le9eq1}
		\sum_{j=1}^{K_n} |(\bw_I)_{1,1,j}^{(L)}| \leq C,
	\end{equation}
	\begin{equation}
		\label{le9eq2}
		|(\bw_I)_{k,i,j}^{(l)}| \leq B \quad (k \in \{1, \dots, K_n\},
		i,j \in \{1, \dots, r\}, l \in \{1, \dots, L-1\})
	\end{equation}
	and
	\begin{equation}
		\label{le9eq3}
		|(\bw_I)_{k,i,j}^{(0)}| \leq A \quad (k \in \{1, \dots, K_n\},
		i \in \{1, \dots, r\}, j \in \{1, \dots,d\}).
	\end{equation}
	Then we have for any $1 \leq p < \infty$, $0 < \epsilon < \beta$ and
	$x_1^n \in [-\alpha,\alpha]^d$
	\begin{eqnarray*}
		&&
		\Nu_p \left(
		\epsilon, \{ T_\beta f  \, : \, f \in \F \}, x_1^n
		\right)
		\\
		&&
		\leq
		\left(
		\nconst \cdot \frac{\beta^p }{\epsilon^p}
		\right)^{
			\nconst \cdot \alpha^{d^*} \cdot A^{d^*} \cdot B^{(L-1) \cdot d^*} \left(\frac{C}{\epsilon}\right)^{d^*/k} + \nconst
		}.\\
	\end{eqnarray*}
\end{lemma}

\begin{proof}
	See Lemma 8 in \citet{KoKr22a}.
\end{proof}

\subsection{Proof of Theorem \ref{th3}}

Let $l=\lfloor \frac{1}{1+d^*} \log_2 n \rfloor$ and $N_n=n^{\nconst}$. Furthermore, let 
\begin{align*}
	\tilde{K}_n= \binom{d}{d^*}(l \cdot
(2^l+1)^{2d^*}+1)^3 \approx \nconst \cdot (\log n)^3 \cdot n^{\frac{6d^*}{1+d^*}}. 
\end{align*}
 Then (\ref{le7eq1}) holds for $\delta= n^{-1/(1+d^*)}$.
Define the weight vector $\bw$ such that the components are chosen according to Lemma 6 so that they approximate $m_I$.

Assumption (\ref{th4eq0}) and (\ref{th4eq1}) of Theorem \ref{th4} are satisfied for $\bw$ since
\begin{align*} 
\sum_{I \subseteq \{1, \dots, d\} \, : \, |I|=d^* } \sum_{k=1}^{\tilde{K}_n} |(\bw_I)_{1,1,k}^{(L)}|^2 \leq \nconst\cdot  n^{-\frac{2d^*}{d^*+1}}
\end{align*}
and
\begin{align*}
\sum_{I \subseteq \{1, \dots, d\} \, : \, |I|=d^* }	\left|  \sum_{k=1}^{\tilde{K}_n} (\bw_I)_{1,1,k}^{(L)} \cdot f_{(\bar{\bw}_I)_{k,1}}^{(L)}(x) \right| \leq \nconst \cdot \left(1+\frac{n^{2d^*/(1+d^*)}}{n^s} \right)
\end{align*}

hold.

Then we obtain by application of Lemma \ref{le8} the assertion of Theorem \ref{th4} for interaction models. Therefore the proof of Theorem \ref{th3} follows similarly to the proof of Theorem \ref{th2}. By applying the assertions of Theorem \ref{th4} and Lemma \ref{le7} we get for $\tau=\frac{1}{1+d^*}$ and $s$ sufficiently large
\begin{eqnarray*}
	&&
	\EXP \int | m_{n}(x)-m(x)|^2 \PROB_X (dx)
	\leq
	\nconst \cdot
	\Bigg(
	\frac{ n^{\frac{1}{1+d^*} \cdot d^* + \epsilon}}{n} + \mmconst{2} \cdot n^{-\frac{2d^*}{1+d^*}}\cdot (\log n)^{4L+3/2}
	\\
	&&
	\hspace*{2cm}
	+
	\sup_{
		(\bar{w}_{i,j,k}^{(l)})_{i,j,k,l} :
		\atop
		|\bar{w}_{i,j,k}^{(l)}-w_{i,j,k}^{(l)}| \leq \log n
		\, (l=0, \dots, L-1)
	}
	\int
	|
	\sum_{k=1}^{\tilde{K}_n}
	w_{1,1,k}^{(L)} \cdot f_{\bar{\bw},k,1}^{(L)}(x)-m(x)|^2 \PROB_X (dx)
	\Bigg)
	\\
	&&
	\leq
	\nconst \cdot
	\Bigg(
	\frac{ n^{\frac{1}{1+d^*} \cdot d^* + \epsilon}}{n} +  \frac{(\log n)^{4L+3/2}}{n^{\frac{2d^*}{1+d^*}}} + \frac{(\log n)^2}{n^{\frac{1}{1+d^*}} } + \frac{1}{n^{\frac{2p}{1+d^*}}} + \frac{\log n \cdot n^{\frac{2d^*}{1+d^*}}}{n^s} 
	\Bigg)
	\\
	&&
	\leq
	\nconst \cdot
	n^{
		- \frac{1}{1+d^*} + \epsilon
	}.
\end{eqnarray*}
\hfill $\Box$
	
	\printbibliography
	
	\appendix

	\section{Proof of Lemma \ref{le2}}
	\begin{proof}
		We have
		\begin{align*}
			\| \nabla_{\bw} F_n (\bw) \|^2
			&
		 =
		\sum_{k,i,j,l}
		\Bigg(
		\frac{2}{n}
		\sum_{s=1}^n
		(Y_s - f_\bw (X_s))
		\cdot
		\frac{\partial f_\bw}{\partial w_{k,i,j}^{(l)}}(X_s)
		\Bigg)^2
		\\
		&
		\leq
		4 \cdot
		\sum_{k,i,j,l}
		\frac{1}{n}
		\sum_{s=1}^n
		(Y_s - f_\bw (X_s))^2 \cdot \left( \frac{\partial f_\bw}{\partial w_{k,i,j}^{(l)}}(X_s) \right)^2
		\\
		&
		\leq
		\nconst \cdot K_n \cdot L \cdot r^2 \cdot d \cdot
		\max_{k,i,j,l,s}       \left(
		\frac{\partial f_\bw}{\partial w_{k,i,j}^{(l)}}(X_s)
		\right)^2
		\cdot
		\frac{1}{n}
		\sum_{s=1}^n
		(Y_s - f_\bw (X_s))^2.
		\end{align*}
		The partial derivative of $f$ is given by
		\begin{align}
		\frac{\partial f_\bw}{\partial w_{k,i,j}^{(l)}}(x)
		=&
		\sum_{s_{l+2}=1}^{r} \dots \sum_{s_{L-1}=1}^{r}
		f_{k,j}^{(l)}(x)
		\cdot
		\sigma^\prime \left(\sum_{t=1}^{r} w_{k,i,t}^{(l)} \cdot f_{k,t}^{(l)}(x) + w_{k,i,0}^{(l)} \right)
		\nonumber \\
		& \quad
		\cdot
		w_{k,s_{l+2},i}^{(l+1)} \cdot
		\sigma^\prime \left(\sum_{t=1}^{r} w_{k,s_{l+2},t}^{(l+1)} \cdot f_{k,t}^{(l+1)}(x) + w_{k,s_{l+2},0}^{(l+1)} \right)
		\cdot
		w_{k,s_{l+3},s_{l+2}}^{(l+2)}
		\nonumber \\
		& \quad
		\cdot
		\sigma^\prime \left(\sum_{t=1}^{r} w_{k,s_{l+3},t}^{(l+2)} \cdot f_{k,t}^{(l+2)}(x) + w_{k,s_{l+3},0}^{(l+2)} \right)
		\cdots
		w_{k,s_{L-1},s_{L-2}}^{(L-2)}
		\nonumber \\
		& \quad
		\cdot
		\sigma^\prime \left(\sum_{t=1}^{r} w_{k,s_{L-1},t}^{(L-2)} \cdot f_{k,t}^{(L-2)}(x) + w_{k,s_{L-1},0}^{(L-2)} \right)
		\cdot
		w_{k,1,s_{L-1}}^{(L-1)}
		\nonumber \\
		& \quad
		\cdot
		\sigma^\prime \left(\sum_{t=1}^{r} w_{k,1,t}^{(L-1)} \cdot f_{k,t}^{(L-1)}(x) + w_{k,1,0}^{(L-1)} \right)
		\cdot
		w_{1,1,k}^{(L)},
		\label{le2proofeq1}
		\end{align}
	where we have used the abbreviations
	\[
	f_{k,j}^{(0)}(x)
	=
	\left\{
	\begin{array}{ll}
		x^{(j)} & \mbox{if } j \in \{1,\dots,d\} \\
		1 & \mbox{if } j=0
	\end{array}
	\right.
	\]
	and
	\[
	f_{k,0}^{(l)}(x)=1 \quad (l=1, \dots, L-1).
	\]
	Together with (\ref{le2eq1}) and (\ref{le2eq2}) we obtain
	\[
	\max_{k,i,j,l,s}       \left(
	\frac{\partial f_\bw}{\partial w_{k,i,j}^{(l)}}(X_s)
	\right)^2
	\leq
	\nconst \cdot r^{2L} \cdot
	\max\{ \|\sigma^\prime\|_\infty^{2L},1\} \cdot B_n^{2L} \cdot (\gamma_n^*)^2 \cdot \alpha_n^2.
	\]

	In the next step of the proof we want to show, that
\[
|f_\bw(x)-f_\bv(x)| \leq 2 \cdot K_n \cdot \max\{\|\sigma'\|_\infty^L,1\} \cdot \gamma_n^* \cdot (2r+1)^L \cdot B_n^L \cdot \alpha_n \cdot \|\bw-\bv\|_\infty \cdot \max\{\|\sigma\|_\infty, 1\}.
\]
Let $\bar{f}_{k,i}^{(l)}$ be defined by
\begin{equation*}
	\bar{f}_{k,i}^{(l)}(x) = \sigma\left(\sum_{j=1}^{r} v_{k,i,j}^{(l-1)} \cdot \bar{f}_{k,j}^{(l-1)}(x) + v_{k,i,0}^{(l-1)} \right)
\end{equation*}
for $l=2, \dots, L$
and
\begin{equation*}
	\bar{f}_{k,i}^{(1)}(x) = \sigma \left(\sum_{j=1}^d v_{k,i,j}^{(0)} \cdot x^{(j)} + v_{k,i,0}^{(0)} \right).
\end{equation*}
First, we show by induction that
\begin{align}\label{le2proofeq2}
	|f_{i,j}^{(l)}(x) - \bar{f}_{i,j}^{(l)}(x)| 
	&\leq \max\{\|\sigma'\|_\infty^l,1\} \cdot (2r+1)^l \cdot B_n^l \cdot \alpha_n   \nonumber\\
	&\qquad \cdot \max_{i,j,s:s<L}|w_{k,i,j}^{(s)}-v_{k,i,j}^{(s)}| \cdot \max\{\|\sigma\|_\infty, 1\}
\end{align}
holds for $l = 1,\dots, L$ and $x \in [-\alpha_n,\alpha_n]^d$.

The function $\sigma$ is differentiable and its derivative is bounded, hence $\sigma$ is Lipschitz continuous with Lipschitz constant $\|\sigma'\|_\infty$. This implies
\begin{align*}
	|f_{i,j}^{(1)}(x) - \bar{f}_{i,j}^{(1)}(x)|&\leq \|\sigma'\|_\infty \cdot \left(\sum_{j=1}^{d}|w_{k,i,j}^{(0)} - v_{k,i,j}^{(0)}|\cdot |x^{(j)}| + |w_{k,i,0}^{(0)} - v_{k,i,0}^{(0)}| \right)\\
	& \leq \|\sigma'\|_\infty \cdot (2r+1) \cdot \alpha_n \cdot \max_{i,j,s:s<L}|w_{k,i,j}^{(s)}-v_{k,i,j}^{(s)}|.
\end{align*}
Assume (\ref{le2proofeq2}) holds for some $l-1$ with $l=2,\dots,L-1$. Then we have
\begin{align*}
	&\left|f_{i,j}^{(l)}(x) - \bar{f}_{i,j}^{(l)}(x)\right|\\
	&\leq \|\sigma'\|_\infty \cdot \Bigg(\sum_{j=1}^{r} \left|w_{k,i,j}^{(l-1)}\right| \cdot \left|f_{k,j}^{(l-1)}(x) - \bar{f}_{k,j}^{(l-1)}(x)\right|\\
	&\hspace{4cm} + \sum_{j=1}^{r}\left|w_{k,i,j}^{(l-1)} - v_{k,i,j}^{(l-1)}\right|\cdot \left|\bar{f}_{k,j}^{(l-1)}(x)\right| + \left|w_{k,i,0}^{(l)} - v_{k,i,0}^{(l-1)}\right|\Bigg) \\
	& \leq \|\sigma'\|_\infty \cdot \Bigg( r \cdot B_n \cdot \max_{j=1,\dots,r} \left|f_{k,j}^{(l-1)}(x) - \bar{f}_{k,j}^{(l-1)}(x)\right|\\
	& \hspace{4cm}+ (r+1) \cdot \max_{i,j,s:s<L}\left|w_{k,i,j}^{(s)} - v_{k,i,j}^{(s)}\right| \cdot \max\{\|\sigma\|_\infty,1\}\Bigg)\\
	& \leq \max\{\|\sigma'\|_\infty^l,1\} \cdot (2r+1)^{l} \cdot B_n^l \cdot\alpha_n \cdot \max_{i,j,s:s<L}\left|w_{k,i,j}^{(s)} - v_{k,i,j}^{(s)}\right|  \cdot \max\{\|\sigma\|_\infty,1\}.
\end{align*}
This implies
\begin{align*}
	&\left|f_{\bw}(x) - f_{\bv}(x)\right|\\
	&= \left|\sum_{j=1}^{K_n} w_{1,1,j}^{(L)} \cdot f_{j,1}^{(L)}(x) - \sum_{j=1}^{K_n} v_{1,1,j}^{(L)} \cdot \bar{f}_{j,1}^{(L)}(x) \right|\\
	&\leq \left|\sum_{j=1}^{K_n} w_{1,1,j}^{(L)}\left(f_{j,1}^{(L)}(x) - \bar{f}_{j,1}^{(L)}(x)\right)\right| + \left|\sum_{j=1}^{K_n} \left(w_{1,1,j}^{(L)}-v_{1,1,j}^{(L)}\right) \cdot \bar{f}_{j,1}^{(L)}(x) \right|\\
	&\leq K_n \cdot \max_{j} \left|w_{1,1,j}^{(L)}\right| \cdot \max_j \left|f_{j,1}^{(L)}(x) - \bar{f}_{j,1}^{(L)}(x)\right|\\
	&\hspace{1cm} + K_n \cdot \max_j \left|w_{1,1,j}^{(L)}-v_{1,1,j}^{(L)}\right| \cdot \max\{\|\sigma\|_\infty, 1\}\\
	&\leq 2 \cdot  K_n  \cdot \gamma_n^* \cdot \max\{\|\sigma'\|_\infty^L,1\} \cdot (2r+1)^L \cdot B_n^L \cdot \alpha_n \cdot \|\bw-\bv\|_\infty \cdot \max\{\|\sigma\|_\infty, 1\}.
\end{align*}
		Together with assumption (\ref{le2eq3}) we can conclude
		\begin{eqnarray*}
			&&
			\frac{1}{n}
			\sum_{s=1}^n
			(Y_s - f_\bw (X_s))^2 
			\\
			&&
			\leq
			2 \cdot F_n(\bv) +       \frac{2}{n}
			\sum_{s=1}^n
			(f_\bv(X_s) - f_\bw (X_s))^2 
			\\
			&&
			\leq 2 \cdot F_n(\bv) + 8 \cdot K_n^2 \cdot (\gamma_n^*)^2 \cdot  \max\{\|\sigma'\|_\infty^{2L},1\} \cdot (2r+1)^{2L} \cdot B_n^{2L} \cdot \alpha_n^2 
			\\
			&&
			\hspace*{2cm}
			\cdot \max\{\|\sigma\|_\infty,1\}^2\cdot \frac{8 t_n}{L_n} \cdot \max\{F_n(\bv) ,1\}.
		\end{eqnarray*}

		By summarizing the above results we obtain the assertion.
	\end{proof}
	
	\section{Proof of Lemma \ref{le3}}
	
	\begin{proof}
		
		We have
		\begin{eqnarray*}
			 &&
			\| \nabla_\bw F_n (\bw_1) -  \nabla_\bw F_n (\bw_2) \|^2
			\\
			&&
			=
			\sum_{k,i,j,l}
			\Bigg(
			\frac{2}{n}
			\sum_{s=1}^n
			(Y_s - f_{\bw_1} (X_s))
			\cdot
			\frac{\partial f_{\bw_1}}{\partial w_{k,i,j}^{(l)}}(X_s)
			\\
			&&
			\quad
			-
			\Bigg(
			\frac{2}{n}
			\sum_{s=1}^n
			(Y_s - f_{\bw_2} (X_s))
			\cdot
			\frac{\partial f_{\bw_2}}{\partial w_{k,i,j}^{(l)}}(X_s)     
			\Bigg)\Bigg)^2
			\\
			&&
			\leq
			8 \cdot
			\sum_{k,i,j,l}
			\left(
			\frac{1}{n}
			\sum_{s=1}^n
			(f_{\bw_2} (X_s) - f_{\bw_1} (X_s))
			\cdot
			\frac{\partial f_{\bw_1}}{\partial w_{k,i,j}^{(l)}}(X_s)
			\right)^2
			\\
			&&
			\quad
			+
			8 \cdot
			\sum_{k,i,j,l}
			\left(
			\frac{1}{n}
			\sum_{s=1}^n
			(Y_s-f_{\bw_2} (X_s) )
			\cdot
			\left(
			\frac{\partial f_{\bw_1}}{\partial w_{k,i,j}^{(l)}}(X_s)
			-
			\frac{\partial f_{\bw_2}}{\partial w_{k,i,j}^{(l)}}(X_s)
			\right)
			\right)^2
			\\
			&&
			\leq
			8 \cdot
			\sum_{k,i,j,l}
			\max_{s=1,\dots,n}
			\left(
			\frac{\partial f_{\bw_1}}{\partial w_{k,i,j}^{(l)}}(X_s)
			\right)^2
			\cdot
			\frac{1}{n}
			\sum_{s=1}^n
			(f_{\bw_2} (X_s) - f_{\bw_1} (X_s))^2\\
			&&
			\quad
			+
			8  \cdot
			\frac{1}{n}
			\sum_{s=1}^n
			(Y_s-f_{\bw_2} (X_s) )^2
			\cdot
			\sum_{k,i,j,l}
			\max_{s=1,\dots,n}
			\left(
			\frac{\partial f_{\bw_1}}{\partial w_{k,i,j}^{(l)}}(X_s)
			-
			\frac{\partial f_{\bw_2}}{\partial w_{k,i,j}^{(l)}}(X_s)
			\right)^2
			.
		\end{eqnarray*}
		From the proof of Lemma \ref{le2} we can conclude
	 \begin{eqnarray*}
		&&
		\sum_{k,i,j,l}
		\max_{s=1,\dots,n}
		\left(
		\frac{\partial f_{\bw_1}}{\partial w_{k,i,j}^{(l)}}(X_s)
		\right)^2\\
		&&
		\leq
		\nconst \cdot K_n \cdot L \cdot r^2 \cdot d \cdot
		r^{2L} \cdot
		\max\{ \|\sigma^\prime\|_\infty^{2L},1 \} \cdot B_n^{2L} \cdot (\gamma_n^*)^2 \cdot \alpha_n^2,      
	\end{eqnarray*}
	\begin{eqnarray*}
		&&
		\frac{1}{n}
		\sum_{s=1}^n
		(f_{\bw_2} (X_s) - f_{\bw_1} (X_s))^2
		\\
		&&
		\leq
		4 \cdot K_n^2 \cdot
		(\gamma_n^*)^2  \cdot \max\{\|\sigma'\|_\infty^{2L},1\}  \cdot (2r+1)^{2L} 
		\cdot B_n^{2L} \cdot \alpha_n^2 \cdot \max\{\|\sigma\|_\infty,1\}^2\|\bw_1-\bw_2\|^2,
	\end{eqnarray*}
		and
		\begin{eqnarray*}
			&&
			\frac{1}{n}
			\sum_{s=1}^n
			(Y_s-f_{\bw_2} (X_s) )^2
			\\
			&&
			\leq
			2 \cdot F_n(\bv)
			+ 8 \cdot  K_n^2 \cdot (\gamma_n^*)^2  \cdot 
			\max\{\|\sigma'\|_\infty^{2L},1\} \cdot (2r+1)^{2L}
			\cdot B_n^{2L} \cdot \alpha_n^2 \cdot \max\{\|\sigma\|_\infty,1\}^2
			\\
			&&
			\quad
			 \frac{8t_n}{L_n} \cdot \max\{F_n(\bv),1\}.
		\end{eqnarray*}
 So it remains to bound
	\[
	\sum_{k,i,j,l}
	\max_{s=1,\dots,n}
	\left(
	\frac{\partial f_{\bw_1}}{\partial w_{k,i,j}^{(l)}}(X_s)
	-
	\frac{\partial f_{\bw_2}}{\partial w_{k,i,j}^{(l)}}(X_s)
	\right)^2.
	\]
	By (\ref{le2proofeq1}) we know that
	\[
	\frac{\partial f_{\bw}}{\partial w_{k,i,j}^{(l)}}(x)
	\]
	 for fixed $x \in [-\alpha_n,\alpha_n]^d$  is a sum of at most $r^{L-2}$ products where each product contains at most $2L+1$ factors. Each of these products contains at most $L$ factors, that are bounded in absolute value by $B_n$ except the last one, which is bounded in absolute value by $\gamma_n^*$.  Considered as a function in $\bw$, these products are Lipschitz continuous with a Lipschitz constant bounded by $1$.

 According to the proof of Lemma \ref{le2} we know that $f_{k,j}^{(l)}(x)$, which is either bounded by $\|\sigma\|_\infty$ or $\alpha_n$, is Lipschitz continuous with a Lipschitz constant bounded by $\max\{\|\sigma'\|_\infty^{l},1\} \cdot \max\{\|\sigma\|_\infty,1\} \cdot (2r+1)^l \cdot B_n^{l} \cdot \alpha_n$. The remaining at most $L$ factors are bounded by $\max\{\|\sigma^\prime\|_\infty,1\}$ with a Lipschitz constant bounded by $\nconst \cdot (2r+1)^L  \cdot B_n^L \cdot \alpha_n \cdot \max\{\|\sigma\|_\infty,1\}$.

The assertion follows from the following result:
If $g_1,\dots, g_s: \mathbb{R} \rightarrow \mathbb{R}$ are Lipschitz continuous functions with Lipschitz constants $C_{Lip,g_1},\dots,C_{Lip,g_s}$, then
\[
\prod_{l=1}^{s} g_l \text{ and } \sum_{l=1}^{s} g_l
\]
are Lipschitz continuous functions with Lipschitz constant bounded by

\[
\sum_{l=1}^{s} C_{Lip,g_l} \cdot \prod_{k \in \{1,\dots,s\}\setminus\{l\}}\|g_k\|_\infty \leq s \cdot \max_l C_{Lip,g_l} \cdot \prod_{k \in \{1,\dots,s\}\setminus\{l\}} \|g_k\|_\infty
\]
and by
\[
\sum_{l=1}^{s} C_{lip,g_l} \leq s \cdot \max_L C_{Lip,g_l}
\]
respectively.

This together with the fact that $\sigma$ and $\sigma'$ are bounded yields
\begin{align*}
	&\sum_{k,i,j,l}
	\max_{s=1,\dots,n}
	\left(
	\frac{\partial f_{\bw_1}}{\partial w_{k,i,j}^{(l)}}(X_s)
	-
	\frac{\partial f_{\bw_2}}{\partial w_{k,i,j}^{(l)}}(X_s)
	\right)^2
 \leq \nconst \cdot  K_n \cdot B_n^{4L} \cdot \alpha_n^4 \cdot (\gamma_n^*)^2  \cdot \|\bw_1-\bw_2\|^2.
\end{align*}

		Summarizing the above results we get the assertion of Lemma \ref{le3}.
		
	\end{proof}
	
\end{document}